\documentclass[11pt]{article}

\usepackage[margin = 1in]{geometry}

\usepackage{tikz}
\usetikzlibrary{decorations.pathreplacing}
\usetikzlibrary{patterns}

\usepackage{tcolorbox}  
\usepackage{mathtools} 

\usepackage[T1]{fontenc}
 
\usepackage[authoryear,round]{natbib}

\usepackage{amsthm,amsmath,amssymb}
 
\theoremstyle{plain}
\newtheorem{theorem}{Theorem}[section]
\newtheorem{proposition}[theorem]{Proposition}
\newtheorem{lemma}[theorem]{Lemma}

\theoremstyle{definition}

\theoremstyle{remark}
\newtheorem{remark}[theorem]{Remark}

\definecolor{darkblue}{rgb}{0.0,0.0,0.65}
\definecolor{darkred}{rgb}{0.68,0.05,0.0}
\definecolor{darkgreen}{rgb}{0.0,0.29,0.29}
\definecolor{darkpurple}{rgb}{0.47,0.09,0.29}
 
\usepackage[backref=page]{hyperref}
\hypersetup{
   colorlinks = true,
   citecolor  = darkblue,
   linkcolor  = darkred,
   filecolor  = darkblue,
   urlcolor   = darkblue,
 }
 \usepackage[capitalize,noabbrev]{cleveref}

\title{Data Augmentation: A Fourier Analysis Perspective}

\author{
Behrooz Tahmasebi\footnote{Harvard John A. Paulson School of Engineering and Applied Sciences, Harvard University, Cambridge, MA 02138, USA. Emails: \texttt{\{behrooz\_tahmasebi,mweber\}@seas.harvard.edu}}
\and
Melanie Weber\footnotemark[1]
\and
Stefanie Jegelka\footnote{Technical University of Munich (CIT, MCML, MDSI) and
MIT Computer Science and Artificial Intelligence Laboratory (CSAIL). Email: \texttt{stefanie.jegelka@tum.de}}
}

\date{}

\begin{document}

\maketitle

\begin{abstract}
Data augmentation is a simple and model-agnostic approach for exploiting known invariances in learning problems. Given a group acting on the input space, one augments the training set with transformed copies of each sample. 
Because it exploits symmetries without modifying the underlying learning algorithm, data augmentation can be applied broadly across learning methods.
However, this universality comes at a computational cost: when the group is large, full group-sized augmentation quickly becomes computationally infeasible.
This raises a fundamental question: \emph{Can partial data augmentation achieve the same statistical benefits as full augmentation in terms of generalization and sample complexity?} We develop a general framework for investigating this question using Fourier analysis and the representation theory of finite groups. We show that, for a broad class of classical learning problems, partial data augmentation based on a randomly sampled subset of group elements achieves the same minimax rates as full augmentation, up to an approximation error that vanishes as the subset size increases. Our results provide a theoretical explanation for why partial augmentation can retain the statistical benefits of full augmentation despite enforcing symmetry only approximately, and shed light on a recently raised question in learning with symmetries \citep{diaz2025invariant}: whether statistically optimal learning under general group invariances can be achieved using computationally scalable methods.
Moreover, we prove a complementary impossibility result: enforcing
\emph{exact} invariance via data augmentation requires averaging over the entire
group, and cannot be achieved by any strict subset when the hypothesis space is
sufficiently expressive.
Together, these results provide a unified perspective on full and partial data augmentation, as well as exact and approximate symmetry enforcement.
\end{abstract}

\vspace{0.5em} \noindent\textbf{Keywords:} data augmentation, invariance, symmetry, sample complexity, representation theory \vspace{0.5em}

 \clearpage
\tableofcontents

\section{Introduction}
\label{sec:intro}

One of the most widely used model-agnostic techniques for exploiting structure
in machine learning is \emph{data augmentation}.
In data augmentation, the training dataset is enriched with transformed copies
of each sample according to known structure inherent in the task.
In learning problems with invariances, this structure is often described by a
group of symmetries acting on the data domain, and augmentation with group
transformations is used to encourage invariance and improve generalization to
unseen data.

Due to its simplicity and broad applicability, data augmentation has become a
standard tool across a wide range of domains, including physics, materials
science, molecular and drug discovery, computer vision, and image processing.
Unlike approaches that enforce invariance through the model architecture, data augmentation exploits symmetries without modifying the underlying learning algorithm.

Despite these advantages, full data augmentation quickly becomes computationally
infeasible when the underlying  group of invariances is large.
This situation arises frequently in practice: for example, permutation and sign-flip groups grow exponentially in size with the data dimension, making
\emph{full} group-sized augmentation prohibitively expensive.
In such settings, practitioners typically resort to \emph{partial} data
augmentation, where only a subset of group elements is used, often chosen
heuristically.
However, the theoretical understanding of when and why partial augmentation
succeeds and whether it can match the statistical benefits of full
augmentation remains limited.

In this paper, we initiate a rigorous study of this question by asking:
\emph{Can partial data augmentation, using a substantially smaller subset of the
group, achieve statistical performance comparable to that of full group
augmentation?}
We address this question in the classical settings of density estimation and
regression using finite-dimensional projection estimators.
Somewhat surprisingly, we show that even very small randomly sampled subsets of
the group suffice to uniformly recover the full statistical benefits of data augmentation,
despite enforcing symmetry only approximately.
Our analysis draws on tools from Fourier analysis on groups and representation theory, and provides a principled explanation for the empirical success of partial data augmentation.

\subsection{Our Contributions}

We summarize the main contributions of this paper.

\paragraph{Statistical optimality of partial data augmentation.}
In Theorem~\ref{thm:partial-augmentation-excess}, we analyze partial data
augmentation for projection-based density and regression estimators.
We prove that partial augmentation using a randomly sampled subset
$S\subseteq G$ achieves the same minimax-optimal sample complexity as full
group-sized augmentation, provided that the number of sampled group elements
satisfies 
\[
|S| \;\gtrsim\; \frac{r}{r_{\mathrm{inv}}},
\]
where \(r\) denotes the dimension of the full feature space and \(r_{\mathrm{inv}}\) is the dimension of the invariant subspace induced by the symmetries.
Indeed, the required size of the augmentation subset depends only on the invariant dimension \(r_{\mathrm{inv}}\) and is independent of the group size. Consequently, statistically optimal rates can be achieved without averaging over the entire group, which may be large or even infinite. This sheds light on a question raised in recent work \citep{diaz2025invariant} concerning whether statistical optimality can be reconciled with computational scalability in learning under general symmetry groups.

\paragraph{Uniform and reusable partial data augmentation.}
In Theorem~\ref{thm:informal-uniform-partial-augmentation} and Theorem~\ref{thm:partial-augmentation-excess-logG-compact}, we further study the
role of partial data augmentation from a uniform generalization perspective.
Specifically, we ask whether a \emph{single} randomly chosen augmentation set
$S$ can be reused across multiple learning tasks and still achieve minimax-optimal
rates with high probability.

Our main finding is that enforcing uniformity over the entire function class
$\mathcal F$ incurs only a mild logarithmic overhead.
Concretely, Theorem~\ref{thm:partial-augmentation-excess-logG-compact} shows that it
suffices to choose
\[
|S| \;\gtrsim\; \frac{r\,\log\!\big(\min\{r,|G|\}\big)}{r_{\mathrm{inv}}},
\]
where $r=\dim(\mathcal F)$ and $r_{\mathrm{inv}}=\dim(\mathcal F^G)$.
Thus, the cost of reusing a single partial augmentation set is only a
$\log(\min\{r,|G|\})$ factor, which remains small even when the function space
dimension $r$ is large or the group $G$ is infinite.

\paragraph{Impossibility of exact invariance via partial augmentation.}
Finally, in Theorem~\ref{thm:informal-exact-invariance}, we establish an
impossibility result highlighting a fundamental computational limitation of
data augmentation.
While partial data augmentation is sufficient for achieving statistical
optimality, we show that enforcing \emph{exact} invariance to a group $G$ is
computationally intractable in general.
Specifically, assuming the hypothesis space is sufficiently rich to represent
all irreducible symmetry modes, exact $G$-invariance via data augmentation
requires averaging over the entire group and thus no strict subset $S\subsetneq G$
can suffice.

\vspace{1em}
Taken together, our results reveal a sharp separation between partial and
full data  augmentation:

\begin{center}
\emph{Partial data augmentation over a small subset of a large symmetry group
is sufficient to attain the full statistical benefits of symmetry, while exact
invariance cannot be guaranteed when augmentation is restricted to a strict
subset of the group.}
\end{center}

\section{Related Work}
\label{sec:rw}

\paragraph{Geometric machine learning and symmetries.}
Geometric machine learning has emerged as a powerful framework for incorporating
symmetries and structure into learning algorithms, with applications spanning
quantum systems, atomistic modeling, continuum mechanics, and beyond
\citep{zhang2025artificial, batzner20223, bronstein2017geometric, smidt2021euclidean, batzner2023advancing,weber2025geometric}.
From a theoretical perspective, the statistical benefits of exploiting symmetries
have been studied for group averaging
\citep{tahmasebi2023exact,tahmasebi2025achieving}, as well as for canonicalization-based
approaches \citep{tahmasebigeneralization}.
Related work has also examined the role of regularization in symmetric models
\citep{tahmasebi2025regularity} and the problem of testing for and identifying invariances
\citep{dehmamy2021automatic,soleymani2025robust,tahmasebi2026adaptive}.
In parallel, recent studies have investigated generalization~\citep{bietti2021sample,mei2021learning} and the computational
complexity~\citep{soleymani2025learning,soleymani2025from,kiani2024hardness} of learning under invariances, highlighting algorithmic barriers that
complement statistical considerations.
Approximation-theoretic guarantees for equivariant learning architectures have
also been developed \citep{petrache2023approximation,pacini2025universality}.

\paragraph{Alternatives to data augmentation.}
While data augmentation is a widely used mechanism for enforcing symmetry,
several works have proposed alternative strategies that encode invariance
directly into the learning algorithm.
Canonicalization methods
\citep{kaba2023equivariance, ma2024canonicalization, dymequivariant, shumaylov2025lie} aim to map inputs to a
canonical representative of their orbit, while frame averaging
\citep{puny2021frame} provides a related approach based on averaging over
structured feature representations.
These methods avoid explicit data augmentation but often require careful design
or additional computational assumptions.

\paragraph{Theoretical perspectives on data augmentation.}
In contrast to the extensive literature on equivariant and invariant models, the
theoretical understanding of data augmentation itself remains comparatively
limited.
Existing work has studied data augmentation from several viewpoints, including
its impact on training dynamics in neural networks \citep{shen2022data}, its role
as an implicit form of regularization \citep{lin2024good, yang2023sample}, and its
group-theoretic foundations \citep{chen2020group}.
Most closely related to our setting, \citet{dao2019kernel} developed a
kernel-based analysis of data augmentation, with further refinements and
extensions in \citep{patil2023generalized, mei2021learning}.
However, these works do not address the question of whether \emph{partial}
augmentation can recover the full statistical benefits of symmetry.

\paragraph{Adaptive and task-driven augmentation.}
Beyond passive augmentation schemes, several recent works have explored
generative, active, or adaptive data augmentation strategies
\citep{zheng2023toward, pmlrv202dong23f, chen2024comprehensive}.
These approaches aim to optimize augmentation policies based on the data or
learning objective, but they are largely orthogonal to the questions studied in
this paper.
For broader perspectives, surveys are available for image augmentation in deep
learning \citep{shorten2019survey}, reinforcement learning
\citep{ma2025comprehensive}, and natural language processing
\citep{li2022data, pellicer2023data}.
Additional application-focused studies include image classification
\citep{mikolajczyk2018data}, graph learning \citep{zhao2022graph}, and other
domains \citep{mumuni2022data}.
It is also important to note that data augmentation can sometimes be detrimental,
as discussed in \citet{kirichenko2023understanding}.
The partial enforcement of symmetry has a long history in applications where
symmetry is either intrinsically approximate
\citep{finzi2021residual,romero2022learning,van2022relaxing,kim2023regularizing,pmlr-v258-park25d,wang2022approximately},
or unknown and therefore must be discovered from data
\citep{yang2024latent,yang2023generative,van2023learning,huh2025a,desai2022symmetry,dehmamy2021automatic,shaw2024symmetry}.

\paragraph{Invariant kernels.}
Recent work on invariant kernels \citep{diaz2025invariant} studies the statistical
and computational properties of learning with symmetry-enforced kernels and
raises open questions about achieving minimax-optimal rates for general groups
without explicitly averaging over the full group.
Our results partially address these questions by showing that partial data
augmentation can recover optimal statistical performance while avoiding full
group-sized averaging.

\section{Problem Statement}
\label{sec:problem-statement}

We formalize the learning problems considered in this paper and set up the
notation used throughout.

\subsection{Data Domain, Symmetry Groups, and Actions}

Let $(\mathcal X,\mu)$ be a measurable space, where $\mu$ denotes a reference
probability measure on $\mathcal X$.
We assume that a group $G$ acts on $\mathcal X$ via measurable maps
\[
(g,x)\;\mapsto\; g x,
\qquad g\in G,\ x\in\mathcal X,
\]
satisfying $e x=x$ and $g(hx)=(gh)x$ for all $g,h\in G$, where $e$ denotes the
identity element.
Throughout the paper, we assume that the action of $G$ on $\mathcal X$ is
\emph{measure-preserving}, meaning that $\mu(gA)=\mu(A)$ for all measurable
sets $A\subseteq\mathcal X$ and all $g\in G$.

This framework captures a wide range of symmetries arising in practice,
including permutations, sign flips, rotations, reflections, and combinations
thereof.
We allow $G$ to be finite or infinite (compact); when sampling from $G$ is required, we
assume there is an oracle access to the uniform sampling for the canonical (Haar) probability measure on the group.

\subsection{Function Spaces and Lifted Group Actions}

Let $\mathcal F\subset L^2(\mathcal X,\mu)$ be a finite-dimensional linear
subspace with $\dim(\mathcal F)=r$.
Let $\Pi_{\mathcal F}$ denote the $L^2(\mathcal{X})$-orthogonal projection onto
$\mathcal F\subset L^2(\mathcal X,\mu)$.
We assume that $\mathcal F$ is \emph{closed under the action of $G$}, meaning that
for every $f\in\mathcal F$ and every $g\in G$, the function
$x\mapsto f(g^{-1}x)$ also belongs to $\mathcal F$.

This closure property induces a lifted action of $G$ on $\mathcal F$.
Specifically, for each $g\in G$, define a linear operator
$T_g:\mathcal F\to\mathcal F$ by
$
(T_g f)(x) := f(g^{-1}x).
$
Under our assumptions, each $T_g$ is unitary with respect to the
$L^2(\mathcal{X})$ inner product, and the map $g\mapsto T_g$ defines a finite-dimensional
unitary representation of $G$ on $\mathcal F$ (Appendix~\ref{app:group_rep}).

A central object in this paper is the subspace of $G$-invariant functions,
defined as
\[
\mathcal F^G
\;:=\;
\{f\in\mathcal F : T_g f = f \ \text{for all } g\in G\}.
\]
We denote its dimension by $r_{\mathrm{inv}} := \dim(\mathcal F^G)$.
Intuitively, $r_{\mathrm{inv}}$ measures the \emph{effective dimension} of the
function class after accounting for the symmetries.

\subsection{Learning Tasks}

We study two classical statistical learning problems.

\paragraph{Density estimation.}
We observe unlabeled samples $x_1,\dots,x_n\in\mathcal X$ drawn i.i.d.\ from an
unknown density $f^\star$ with respect to $\mu$, where $f^\star\in L^2(\mathcal X,\mu)\cap L^\infty(\mathcal X,\mu)$.
The goal is to estimate $f^\star$ in squared $L^2(\mathcal{X})$ risk.

\paragraph{Supervised regression.}
We observe labeled samples $\{(x_i,y_i)\}_{i=1}^n$ generated according to
$
y_i = f^\star(x_i) + \varepsilon_i,
$
where $x_i\sim\mu$ i.i.d., the noise variables $\varepsilon_i$ are independent,
mean-zero, and satisfy $\mathbb E[\varepsilon_i^2]=\sigma^2$.
The regression function $f^\star$ is assumed to belong to
$L^2(\mathcal X,\mu)\cap L^\infty(\mathcal X,\mu)$.
Performance is measured in squared $L^2(\mathcal{X})$ error.

\begin{remark}[Approximation bias]
We do \emph{not} assume that $f^\star\in\mathcal F$. Rather, $\mathcal F$
serves only as the hypothesis class used for estimation, so the best achievable
target in $L^2(\mathcal X,\mu)$ is the projection of $f^\star$ onto
$\mathcal F$.
This yields an approximation bias, which can be decreased by enlarging $\mathcal F$, at the cost of increased variance. Thus, as usual, the choice of $\mathcal F$ reflects a bias--variance trade-off.
\end{remark}

For both settings, we focus on \emph{projection estimators}, which estimate
$f^\star$ by projecting empirical moments onto the finite-dimensional space
$\mathcal F$.
These estimators are classical, minimax-optimal, 
and serve as a clean testbed for studying the effect of symmetry and
data augmentation.

\subsection{Data Augmentation}

Given a set of group elements $S\subseteq G$, data augmentation proceeds by
transforming each observed sample using elements of $S$.
Concretely, the augmented sample is given by
\[
\{g^{-1} x_i : i\in[n],\ g\in S\},
\qquad
\{(g^{-1} x_i, y_i) : i\in[n],\ g\in S\},
\]
in the density estimation and regression settings, respectively.
We distinguish two regimes:
\begin{itemize}
\item \emph{Full data augmentation}, where $S=G$.
\item \emph{Partial data augmentation}, where $S$ is a (typically random)
strict subset of $G$.
\end{itemize}

Full augmentation can enforce exact invariance but is often computationally
infeasible when $G$ is large or infinite.
Partial augmentation is computationally efficient and  more practical,
but its statistical benefits are less well understood.

\subsection{Objectives and Questions}

Our goal is to understand the trade-offs between statistical efficiency,
computational cost, and symmetry enforcement when using partial data
augmentation. We address the following questions:
\begin{itemize}
\item \textbf{Statistical efficiency.}
Can partial data augmentation achieve the same minimax-optimal rates as full
data augmentation for density estimation and regression?

\item \textbf{Reusability and uniformity.}
Can a single randomly chosen augmentation set $S$ be reused across tasks or
estimators, while still providing uniform generalization guarantees?

\item \textbf{Exact versus approximate invariance.}
Is it possible to enforce exact $G$-invariance via partial data augmentation, or
is full group-sized augmentation fundamentally necessary?
\end{itemize}

\subsection{Projection Estimators}
\label{sec:projection-estimators}

A central object in this paper is the class of \emph{projection estimators} for
density estimation and regression.
These estimators are classical, minimax-optimal over finite-dimensional
function classes, and provide a transparent setting for investigating the effect of
symmetry and data augmentation. Extensions to ordinary least squares and infinite-dimensional hypothesis classes
are discussed in Appendix~\ref{app:ols-and-infinite-dimensional}.

\paragraph{Density estimation.} Let $(\phi_\ell)_{\ell=1}^r$ be a fixed
$L^2(\mathcal{X})$-orthonormal basis of $\mathcal F$.
Suppose we observe unlabeled samples $x_1,\dots,x_n$ drawn i.i.d.\ from an unknown
density $f^\star$ with respect to $\mu$, where
$f^\star\in L^2(\mathcal X,\mu)\cap L^\infty(\mathcal X,\mu)$.
The population projection of $f^\star$ onto $\mathcal F$ is
\[
\Pi_{\mathcal F} f^\star
=
\sum_{\ell=1}^r
\theta_\ell\,\phi_\ell,
\qquad
\theta_\ell := \mathbb E[\phi_\ell(x)].
\]
The coefficients $\theta_\ell$ can be estimated unbiasedly from data by empirical
averages, leading to the projection density estimator
\[
\widehat\theta_\ell := \frac{1}{n}\sum_{i=1}^n \phi_\ell(x_i)
\implies 
\widehat f := \sum_{\ell=1}^r \widehat\theta_\ell\,\phi_\ell \;\in\; \mathcal F.
\]

\paragraph{Regression.}
In the regression setting, we observe labeled samples $(x_i,y_i)_{i=1}^n$ with
\[
y_i = f^\star(x_i) + \varepsilon_i,
\qquad
\mathbb E[\varepsilon_i]=0,
\quad
\mathbb E[\varepsilon_i^2]=\sigma^2.
\]
The population projection of the regression function $f^\star$ onto $\mathcal F$
is again $\Pi_{\mathcal F} f^\star$, with coefficients
\[
\beta_\ell := \mathbb E[y\,\phi_\ell(x)].
\]
Estimating these moments empirically yields the projection regression estimator
\[
\widehat f := \sum_{\ell=1}^r \widehat\beta_\ell\,\phi_\ell,
\qquad
\widehat\beta_\ell := \frac{1}{n}\sum_{i=1}^n y_i\,\phi_\ell(x_i).
\]

\paragraph{Statistical properties.}
Projection estimators achieve minimax-optimal rates over $r$-dimensional 
classes, with expected $L^2(\mathcal{X})$ error of order $r/n$ in both density estimation
and regression.
Moreover, their linear structure makes them particularly amenable to analysis
under group actions and data augmentation, as averaging over group
transformations corresponds to linear operators acting on the coefficient
representation.

For these reasons, projection estimators serve as a canonical and analytically
tractable setting for studying the statistical role of partial and full data
augmentation under symmetries.

\begin{remark}[Black-box augmentation]
A key caveat is that we study partial augmentation as a \emph{black-box}
mechanism for promoting symmetry. If the estimator itself can be modified, then
symmetry may instead be incorporated directly into classical projection
estimation via group-constrained optimization, which can be done in polynomial
time \citep{soleymani2025learning}; such procedures are not
black-box.
Thus, to isolate the role of data augmentation, we decouple invariance from the
estimator and treat augmentation as a model-agnostic preprocessing step.
Projection estimators provide a clean testbed for this purpose: the estimator is
fixed, and the subset $S\subseteq G$ is the only mechanism used to promote
symmetry. This lets us study how $S$ controls statistical gains and the
transition from approximate to exact invariance.
\end{remark}

\begin{remark}[Augmentation and group averaging]
Projection estimators are \emph{linear} in the dataset, meaning that if
$\mathcal D_1$ and $\mathcal D_2$ are two datasets and
$\mathcal D=\mathcal D_1\cup \mathcal D_2$, then
$
    \widehat f_{\mathcal D}
    =
    \frac{|\mathcal D_1|}{|\mathcal D|}\widehat f_{\mathcal D_1}
    +
    \frac{|\mathcal D_2|}{|\mathcal D|}\widehat f_{\mathcal D_2}.
$ Hence, for projection estimators, augmenting the dataset by $S\subseteq G$
is equivalent to applying the subset-averaging operator
$
    \Pi_S f(x):=\frac{1}{|S|}\sum_{g\in S} f(g^{-1}x)
$
to the unaugmented estimator $\widehat f$. In this sense, data augmentation
as preprocessing and subset-group averaging as post-processing are dual ways of
imposing approximate symmetry. Details on group averaging are provided in
Appendix~\ref{app:approx-projection-random-averaging}.
\end{remark}

\section{Main Results}\label{sec:mr}

We begin by studying the statistical effect of partial data augmentation for
classical projection-based estimators.
Our first main result shows that, for projection estimators in both density
estimation and regression, partial data augmentation is sufficient to recover
the full statistical gains of symmetry, up to a controllable approximation
error that depends only on the size of $S$.

\begin{theorem}[Partial data augmentation for projection estimators]
\label{thm:partial-augmentation-excess}
Let $\mathcal{F}\subset L^2(\mathcal X,\mu)$ be a finite-dimensional space of
dimension $r$, closed under the action of a group $G$, and let
$\mathcal{F}^G$ denote its invariant subspace with dimension $r_{\mathrm{inv}}$.
Assume $f^\star\in L^2(\mathcal X)\cap L^\infty(\mathcal X)$.
Let $x_1,\dots,x_n\sim\mu$ be i.i.d., and let $S=\{g_1,\dots,g_{|S|}\}$ be i.i.d.\
uniform samples from $G$. Let $\widehat f_S$ be the projection estimator obtained
by augmenting each sample $x_i$ by $\{g x_i:g\in S\}$.
Then, the expected \emph{excess} $L^2(\mathcal{X})$ error over $\mathcal{F}^G$ satisfies
\[
\mathbb E\big[\|\widehat f_S - \Pi_{\mathcal{F}^G} f^\star\|_{L^2(\mathcal{X})}^2\big]
\le
C\left(
\frac{\|f^\star\|_\infty}{n}\,r_{\mathrm{inv}}
+
\frac{r}{n|S|}
\right),
\]
for some absolute constant $C>0$. The same holds for projection regression
estimators, with $\|f^\star\|_\infty r_{\mathrm{inv}}/n$ replaced by
$(\|f^\star\|_\infty^2+\sigma^2)r_{\mathrm{inv}}/n$. Moreover, this upper bound is tight up to absolute constants.
\end{theorem}

The bound in Theorem~\ref{thm:partial-augmentation-excess} decomposes the excess
risk into two terms with clear interpretations.
The first term, \(r_{\mathrm{inv}}/n\), corresponds to the estimation error
within the invariant subspace \(\mathcal F^G\) and matches the minimax-optimal rate
achieved by full data augmentation. The second term, \(r/(n\,|S|)\), quantifies the error due to using only a subset
of group elements and vanishes as the size of the partial augmentation set increases.

In particular, as soon as $|S| \gtrsim r/r_{\mathrm{inv}}$, the contribution of
partial augmentation becomes negligible, and the estimator achieves the same
statistical performance as if full group-sized augmentation were used.
Notably, this guarantee holds regardless of the size of the group $G$, which may
be exponentially large or infinite.
Thus, partial data augmentation can retain the full statistical benefits of
symmetry while dramatically reducing computational cost; beyond this point,
additional augmentation provides only redundant information.

The guarantees in Theorem~\ref{thm:partial-augmentation-excess} control the
performance of partial data augmentation for a fixed estimator, in expectation
over the random choice of the augmentation set $S$.
In many settings, however, one would like to draw a single augmentation set once
and reuse it across multiple learning tasks, datasets, or algorithms.
This raises a stronger question: \emph{Can partial data augmentation provide
\emph{uniform} guarantees that hold simultaneously for all functions in the
hypothesis space, with high probability over the choice of $S$?}

Our next result answers this question in the affirmative.
It shows that a single randomly chosen augmentation set $S$ suffices to
approximate full data augmentation uniformly over the entire function class,
at the cost of only a mild logarithmic factor.

\begin{theorem}[Informal version of Theorem~\ref{thm:partial-augmentation-uniform-highprob}: uniform partial data augmentation]
\label{thm:informal-uniform-partial-augmentation} 
With high probability over the choice of a random augmentation set $S\subseteq G$ of size $|S|$, partial data augmentation using $S$
provides a uniform approximation to full data augmentation over the entire
hypothesis space $\mathcal F$.
Specifically, for all functions $f\in\mathcal F$ with bounded $L^2(\mathcal{X})$ norm,
the output obtained by averaging $f$ using $S$ is close to the output obtained
by averaging using the full group $G$, with approximation error on the order of
$\displaystyle \sqrt{\frac{\log(\min\{r,|G|\})}{|S|}}$,
where $r=\dim(\mathcal F)$.
As a consequence, a single randomly chosen augmentation set $S$ can be reused
across all learning algorithms whose outputs lie in $\mathcal F$, incurring only
a vanishing worst-case error as $|S|$ grows.
\end{theorem}

Theorem~\ref{thm:informal-uniform-partial-augmentation} shows that partial data
augmentation can be made \emph{reusable}:
once a random subset $S$ of group elements is sampled, it can be applied across
all predictors in $\mathcal F$ without sacrificing statistical guarantees.
Compared to the expectation-based bounds of
Theorem~\ref{thm:partial-augmentation-excess}, the price of uniformity is only a  
$\log(\min\{r,|G|\})$ factor. 

In particular, if $|S|$ grows slightly faster than
$\log(\min\{r,|G|\}$, the approximation error due to partial augmentation
becomes negligible.
This result provides a theoretical justification for practical pipelines in
which a single, fixed augmentation set is reused across tasks or models, and
highlights a distinction between expected and uniform guarantees for
partial data augmentation.

\paragraph{Proof sketch for Theorem~\ref{thm:informal-uniform-partial-augmentation}:} At a high level, the proof views data augmentation through its dual
subset-averaging operator. Representation theory (Fourier analysis over the group) decomposes this operator into
harmonic components, or symmetry modes, corresponding to the group action on
$\mathcal F$. Full augmentation removes all nontrivial modes, while partial
augmentation only approximately cancels them. For a random subset $S$, each mode
is controlled by a large-deviation bound, and a union bound over the relevant
modes gives the logarithmic factor $\log(\min\{r,|G|\})$.

We now specialize this uniform guarantee to the concrete setting of projection
estimators for density estimation and regression.
This allows us to translate uniform approximation of augmentation operators
directly into high-probability excess risk bounds for learning, while retaining
the computational advantages of partial augmentation.

In particular, the next theorem shows that a \emph{single} randomly chosen
augmentation set $S$ can be reused for projection-based learning, yielding
minimax-optimal rates up to a logarithmic factor that quantifies the cost of
uniformity.

\begin{theorem}[Uniform partial data augmentation for projection estimators]
\label{thm:partial-augmentation-excess-logG-compact}
Let $\mathcal{F}\subset L^2(\mathcal X,\mu)$ be a finite-dimensional space of
dimension $r$, closed under the action of a group $G$, and let
$\mathcal{F}^G$ denote its invariant subspace with dimension $r_{\mathrm{inv}}$.
Assume $f^\star\in L^2(\mathcal{X})\cap L^\infty(\mu)$.
Let $x_1,\dots,x_n\sim\mu$ be i.i.d., and let
$S=\{g_1,\dots,g_{|S|}\}$ be i.i.d.\ uniform samples from $G$.
Let $\widehat f_S$ denote the projection estimator obtained via
\emph{partial data augmentation} using $S$.

Then, for any $\delta\in(0,1)$, with probability at least $1-\delta$ over the
draw of $S$, the following bound holds simultaneously for all
$f^\star\in\mathcal F^G$ with $\|f^\star\|_\infty\le B$:
\[
\mathbb E\!\left[\big\|\widehat f_S-\Pi_{\mathcal F^G}f^\star\big\|_{L^2(\mathcal{X})}^2
\;\middle|\; S\right]
\le
C \left(
\frac{r_{\mathrm{inv}}}{n}
+
\frac{r \log\!\big(\min\{r,|G|\}/\delta\big)}{n|S|}
\right),
\]
where $C>0$ depends only on $B$. The same bound holds for projection
regression estimators under additive zero-mean noise with variance $\sigma^2$,
with $C$ also depending on $\sigma^2$.

\end{theorem}

Theorem~\ref{thm:partial-augmentation-excess-logG-compact} refines
Theorem~\ref{thm:partial-augmentation-excess} by providing a high-probability,
reusable guarantee for partial data augmentation.
In particular, the augmentation set \(S\) is reusable: on the same
high-probability event, the bound holds uniformly even if
\(f^\star\in\mathcal F^G\) is chosen adversarially after observing \(S\).
Compared to the expectation-based bound, the only additional cost is a
$\log(\min\{r,|G|\})$ factor, which arises from enforcing uniform control over the
entire function space $\mathcal F$.
Importantly, this logarithmic dependence remains mild even when the group $G$ is
large or infinite, since its dependence on the ambient dimension of the data can be benign.
Thus, uniform partial data augmentation simultaneously achieves statistical
efficiency, reusability, and computational scalability.

\begin{remark}[Random versus structured augmentation sets]
Although group-specific designs for the subset $S\subseteq G$ can sometimes yield
sharper bounds than a random choice, improving over the logarithmic factor
$\log |G|$, constructing such designs is often combinatorial and remains an
active area of research; see, e.g., \citep{v009a015, bourgain2008uniform}. Our results show that random augmentation sets already provide
uniform guarantees with only logarithmic dependence on $|G|$. This relies on the
expansion properties of random subsets of groups \citep{alon1994random}.
\end{remark}

The preceding results demonstrate that partial data augmentation is sufficient
to achieve the full \emph{statistical} benefits of symmetry, both in expectation
and uniformly with high probability.
This naturally raises a complementary question: \emph{Can partial data augmentation
also be used to enforce \emph{exact} invariance under a symmetry group?}

Our final result shows that this is fundamentally impossible in general.
While partial augmentation can approximate full augmentation arbitrarily well,
exact invariance places a much stronger requirement that cannot be met without
averaging over the entire group.

\begin{theorem}[Informal version of Theorem~\ref{thm:exact-invariance-forces-full-group}: exact invariance requires full augmentation]
\label{thm:informal-exact-invariance}
For finite groups, partial data augmentation cannot enforce \emph{exact}
$G$-invariance unless the full group is used.
More precisely, if a learning procedure based on averaging over a subset
$S\subseteq G$ produces outputs that are exactly invariant under all elements
of $G$, and if the hypothesis space is rich enough to represent all irreducible
symmetry modes of $G$, then $S$ must coincide with the full group.
Equivalently, exact symmetry enforcement via data augmentation fundamentally
requires averaging over all group elements; partial augmentation can only yield
approximate invariance.
\end{theorem}

Theorem~\ref{thm:informal-exact-invariance} establishes a sharp separation
between approximate and exact symmetry enforcement via data augmentation.
In contrast to the statistical guarantees obtained with partial augmentation,
exact $G$-invariance requires full group-sized augmentation whenever the
hypothesis space is sufficiently expressive.

 \begin{remark}
The preceding result should be interpreted within the black-box augmentation
framework described earlier. It does not preclude specialized methods that exploit
the group structure directly and bypass augmentation altogether
\citep{soleymani2025learning}. Rather, our aim here is to isolate the effect of
the augmentation subset $S$ when the estimator is fixed.
\end{remark}

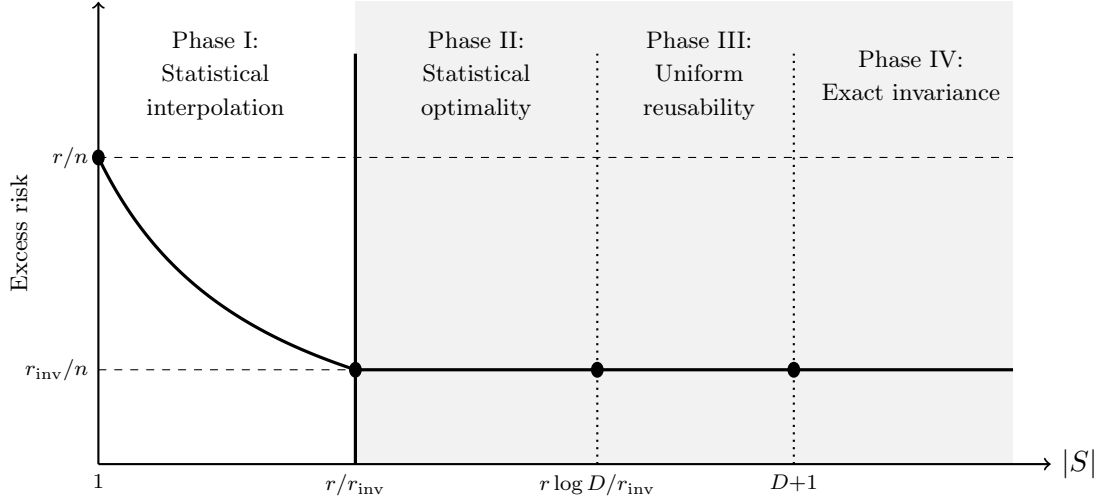
\begin{figure}[t]
\centering
\begin{tikzpicture}[xscale=1.0,yscale=1.25]

    \tikzset{
        threshold/.style={dotted, thick},
        statthreshold/.style={solid, very thick},
        riskcurve/.style={very thick}
    }

    \def\xmax{10.6}
    \def\ymax{4.9}
    \def\ytop{3.25}
    \def\yflat{1.00}
    \def\xstat{3.4}
    \def\xunif{6.6}
    \def\xexact{9.2}
    \def\alpha{1.5} 
    
\fill[gray!10] (\xstat,0) rectangle (\xmax + 1.5,\ymax);
\node[anchor=north west] at (\xstat+0.15,3.95)
    {};

    \draw[->, thick] (0,0) -- (\xmax+2,0) node[right] {\(|S|\)};
    \draw[->, thick] (0,0) -- (0,\ymax);
    \node[rotate=90] at (-1.05,2.45) {\footnotesize Excess risk};

    \draw[dashed] (0,\ytop) -- (\xmax+1.5,\ytop);
    \node[left] at (0,\ytop) {\( \scriptstyle r/n\)};

    \draw[dashed] (0,\yflat) -- (\xmax-0.2,\yflat);
    \node[left] at (0,\yflat) {\(\ \scriptstyle r_{\mathrm{inv}}/n\)};

    \draw[riskcurve, domain=0:\xstat, samples=120, smooth]
        plot (\x, { \yflat + (\ytop-\yflat) * (1-\x/\xstat)/(1+\alpha*\x/\xstat) });

    \draw[riskcurve] (\xstat,\yflat) -- (\xmax+1.5,\yflat);

    \draw[statthreshold] (\xstat,0) -- (\xstat,4.35);
    \draw[threshold] (\xunif,0) -- (\xunif,4.35);
    \draw[threshold] (\xexact,0) -- (\xexact,4.35);

    \fill (0,\ytop) circle (2.4pt);
    \fill (\xstat,\yflat) circle (2.4pt);
    \fill (\xunif,\yflat) circle (2.4pt);
    \fill (\xexact,\yflat) circle (2.4pt);

    \node[below] at (0,0) {\(\scriptstyle 1\)};
    \node[below] at (\xstat,0) {\(\scriptstyle r/r_{\mathrm{inv}}\)};
    \node[below] at (\xunif,0) {\(\scriptstyle r\log D/r_{\mathrm{inv}}\)};
    \node[below] at (\xexact,0) {\(\scriptstyle D+1\)};

    \node[align=center] at (1.55,4.12)
        {\footnotesize Phase I:\\[-1pt]\footnotesize Statistical \\[-1pt]\footnotesize interpolation};

    \node[align=center] at (5.00,4.12)
        {\footnotesize Phase II:\\[-1pt]\footnotesize Statistical \\[-1pt]\footnotesize optimality};

    \node[align=center] at (7.95,4.12)
        {\footnotesize Phase III:\\[-1pt]\footnotesize Uniform \\[-1pt]\footnotesize reusability};

    \node[align=center] at (10.75,4.12)
        {\footnotesize Phase IV:\\[-1pt]\footnotesize Exact invariance};

\end{tikzpicture}
\caption{
Phase diagram for partial augmentation. The risk decreases from the
ordinary rate \(r/n\) to the invariant rate \(r_{\mathrm{inv}}/n\), and
saturates once \(|S|\asymp r/r_{\mathrm{inv}}\). Larger augmentation sets may
still be required for stronger guarantees, such as uniform reusability and
exact invariance.
}
\label{fig:augmentation-phase-diagram}
\end{figure}

\section{Interpreting the Regimes of Partial Augmentation}
\label{subsec:phase-transition}

Our results reveal several distinct regimes for partial data augmentation. These
regimes are all governed by the size of the augmentation set \(S\), but they
correspond to different goals: statistical optimality, uniform reusability, and
exact invariance. A key message is that these goals do not coincide. In
particular, the statistical risk can already match that of full augmentation
well before the augmentation operator is reusable uniformly over the whole
function class, and far before it enforces exact invariance.

To make this distinction explicit, consider the projection-estimator setting
with \(r=\dim(\mathcal F)\) and
\(r_{\mathrm{inv}}=\dim(\mathcal F^G)\). Without augmentation, the estimation
error scales as \(r/n\). Full augmentation reduces the effective dimension to
\(r_{\mathrm{inv}}\), yielding the invariant rate \(r_{\mathrm{inv}}/n\).
Partial augmentation interpolates between these two rates. Indeed, the
augmentation error is controlled by
$
    \frac{r}{n|S|},
$
so the total error behaves as
$
    \frac{r_{\mathrm{inv}}}{n}
    +
    \frac{r}{n|S|}.
$
Thus, when \(|S|\) is small, partial augmentation only partially suppresses the
non-invariant directions, and the risk lies between the ordinary rate \(r/n\)
and the invariant rate \(r_{\mathrm{inv}}/n\). The first critical threshold is
therefore
\[
    |S|_{\mathrm{stat}}
    \asymp
    \frac{r}{r_{\mathrm{inv}}}.
\]
Once \(|S|\gtrsim r/r_{\mathrm{inv}}\), the augmentation error is of the same
order as, or smaller than, the invariant estimation error. In this regime,
partial augmentation is statistically indistinguishable from full augmentation,
up to constants. This threshold is optimally characterized by our upper  bound: below this scale, the additional augmentation error is unavoidable,
while above it the minimax rate saturates at the invariant rate. Thus,
\(r/r_{\mathrm{inv}}\) marks the statistical phase transition, beyond which
increasing \(|S|\) no longer improves the minimax rate.

A stronger requirement is uniform reusability. Rather than controlling the risk
of a fixed estimator in expectation, one may want a single sampled set \(S\) to
work uniformly over all functions in \(\mathcal F\). This requires high-probability
control of the operator norm \(\|\Pi_S-\Pi_G\|_{\mathrm{op}}\). Our uniform
bound shows that this incurs only a logarithmic overhead:
\[
    \|\Pi_S-\Pi_G\|_{\mathrm{op}}^2
    \lesssim
    \frac{\log N}{|S|},
    \qquad
    N:=\min\{r,|G|\}.
\]
More sharply, \(N\) may be replaced by a representation-theoretic Fourier
complexity
\[
    D
    :=
    \sum_{\lambda\in\Lambda} d_\lambda^2 \le \min\{r, |G|-1\},
\]
where \(\Lambda\) denotes the set of nontrivial irreducible representations of
\(G\) that appear in \(\mathcal F\), and \(d_\lambda\) is the dimension of the
irrep \(\lambda\). For background on group representations, see Appendix~\ref{app:group_rep}. 

Consequently, the reusable-uniform regime begins once
\[
    |S|
    \gtrsim
    \frac{r}{r_{\mathrm{inv}}}\log (\min\{N,D\}),
\]
Equivalently, the transition size satisfies the upper bound 
\[ |S|_{\mathrm{unif}} \lesssim \frac{r}{r_{\mathrm{inv}}}\log(\min\{N,D\}). \]

This logarithmic factor is the price of reusing the same random augmentation set
uniformly over many Fourier modes. It is unavoidable in worst-case
representations (e.g., in case of commutative groups), although the exact instance-wise threshold can depend on finer
group-theoretic structure.

Finally, exact invariance is a still stronger requirement. Statistical
optimality and uniform reusability require \(\Pi_S\) to approximate \(\Pi_G\),
whereas exact invariance requires equality of averaging operators:
\[
    \Pi_S=\Pi_G
    \qquad
    \text{on } \mathcal F .
\]
This can force \(S=G\) when \(\mathcal F\) is
sufficiently expressive, for example when it contains all irreducible symmetry
modes. Thus exact symmetry enforcement can be much more expensive than
statistical optimality.
Moreover, if weighted augmentation is allowed, one can obtain a general upper bound using
Carathéodory's theorem. Let the complexified representation of \(G\) on
\(\mathcal F\) decompose as
\[
    \mathcal F_{\mathbb C}
    \cong
    \mathcal F^G_{\mathbb C}
    \oplus
    \bigoplus_{\lambda\in\Lambda}
    \mathbb C^{m_\lambda}\otimes V_\lambda ,
\]
where \(\Lambda\) contains only the nontrivial irreducible representations
appearing in \(\mathcal F\). A weighted augmentation rule
\[
    \nu=\sum_{s\in S} w_s\delta_s,
    \qquad
    w_s\ge 0,
    \qquad
    \sum_{s\in S} w_s=1,
\]
satisfies \(\Pi_\nu=\Pi_G\) on \(\mathcal F\) if and only if
\[
    \sum_{s\in S} w_s \rho_\lambda(s)=0
    \qquad
    \text{for every } \lambda\in\Lambda .
\]
Equivalently, the origin must be written as a convex combination of the Fourier
feature vectors
\[
    g
    \longmapsto
    \big(\rho_\lambda(g)\big)_{\lambda\in\Lambda}.
\]
These vectors lie in a real vector space of dimension at most
$
    D
    =
    \sum_{\lambda\in\Lambda} d_\lambda^2,
$
up to the standard realification of complex matrix coefficients. Since Haar
averaging places the origin in their convex hull, Carathéodory's theorem implies
that there exists a weighted exact augmentation rule with
\[
    |S|_{\mathrm{exact}}
    \le
    D+1.
\]
This is a bound for weighted augmentation only; it does not imply the existence
of an unweighted subset of the same size. Moreover, the exact threshold is
problem-dependent. It can be as large as \(|G|\) in representation-complete
settings, while it can be much smaller in special incomplete representations.

In summary, partial augmentation exhibits a hierarchy of increasingly stringent
requirements:
\[
    |S|_{\mathrm{stat}}
    \;\lesssim\;
    |S|_{\mathrm{unif}}
    \;\lesssim\;
    |S|_{\mathrm{exact}} .
\]
The first threshold governs statistical optimality, the second governs uniform
reusability, and the third governs exact invariance. The statistical risk,
however, already saturates at the first threshold:
\[
    \mathbb E\|\widehat f_S-\Pi_{\mathcal F^G}f^\star\|_{L^2(\mathcal X)}^2
    \asymp
    \frac{r_{\mathrm{inv}}}{n}.
\]
The later thresholds provide stronger guarantees, but they do
not improve the statistical rate.

\subsection{Extension to Binary Classification}
\label{subsec:classification}

Although our results are stated for regression and density estimation, they also
immediately imply guarantees for binary classification through a standard
plug-in argument. We assume that \((X,Y)\sim P\), where \(Y\in\{\pm1\}\), and the
optimal prediction rule is the Bayes classifier induced by the regression
function
\[
    f^\star(x)
    :=
    \mathbb E[Y\mid X=x].
\]
The Bayes classifier is
\[
    y^\star(x)
    :=
    \operatorname{sign}(f^\star(x)).
\]
Given an estimator \(\widehat f\) of \(f^\star\), we form the plug-in classifier
\[
    \widehat y(x)
    :=
    \operatorname{sign}(\widehat f(x)) \in \{\pm 1\}.
\]
Then the excess classification risk is controlled by the regression error:
\[
    \mathbb P(\widehat y(X)\neq Y)
    -
    \mathbb P(y^\star(X)\neq Y)
    \le
    \mathbb E_X\big[|\widehat f(X)-f^\star(X)|\big]
    \le
    \|\widehat f-f^\star\|_{L^2(\mathcal X)} .
\]
Indeed, the excess risk can be written as
\[
    \mathbb E_X
    \Big[
        |f^\star(X)|
        \mathbf 1\{\widehat y(X)\neq y^\star(X)\}
    \Big].
\]
On the event \(\widehat y(X)\neq y^\star(X)\), the signs of
\(\widehat f(X)\) and \(f^\star(X)\) disagree, and hence
\[
    |f^\star(X)|
    \le
    |\widehat f(X)-f^\star(X)|.
\]
Thus, the claimed bound follows.
Consequently, the regression guarantees in this paper directly imply
classification guarantees for the corresponding plug-in classifiers. For example, if \(\widehat f_S\) denotes the partially augmented projection
estimator and the target lies in the invariant class, then the same argument
used for our regression bounds shows that
\[
    \mathbb E
    \big[
        \|\widehat f_S-f^\star\|_{L^2(\mathcal X)}^2
    \big]
    \lesssim
    \frac{r_{\mathrm{inv}}}{n}
    +
    \frac{r}{n|S|}
\]
which implies
\[
    \mathbb E
    \Big[
        \mathbb P(\widehat y_S(X)\neq Y)
        -
        \mathbb P(y^\star(X)\neq Y)
    \Big]
    \lesssim
    \left(
        \frac{r_{\mathrm{inv}}}{n}
        +
        \frac{r}{n|S|}
    \right)^{1/2}.
\]
Thus, the same augmentation regimes appear in binary classification: partial
augmentation interpolates between the ordinary and invariant rates, and once
\(|S|\gtrsim r/r_{\mathrm{inv}}\), the plug-in classifier achieves the same
classification rate as the fully augmented estimator, up to constants. Moreover,
the same uniform reusability guarantee applies: a single randomly chosen
augmentation set \(S\) can be reused across classifiers whose underlying
regression estimates lie in \(\mathcal F\).

\section{Projection Estimators on the Sphere}
\label{sec:sphere-projection}

We briefly specialize our framework to the unit sphere
$\mathbb S^{d-1}$, where projection estimators admit a simple and efficient
implementation via spherical harmonics.
This setting illustrates that our results apply naturally to symmetry subgroups of the orthogonal group
$G\subseteq \mathrm{O}(d)$ and that partial data augmentation is computationally tractable,
in contrast to full group-sized augmentation. 
For more details and explanations on applications to the sphere, see
Appendices~\ref{apd:bck:sh}, \ref{apd:bck:sh2}, and \ref{apd:bck:sh3}.

Let $\mathcal X=\mathbb S^{d-1}$ with $\mu$ the uniform probability measure.
For each degree $\ell\ge 0$, let $\mathcal H_\ell$ denote the space of spherical
harmonics of degree $\ell$, with dimension $N_\ell$.
For a cutoff $k\ge 0$, define the truncated space
$
\mathcal F := \bigoplus_{\ell=0}^{k}\mathcal H_\ell$ with
$r := \dim(\mathcal F)=\sum_{\ell=0}^k N_\ell
$,
and let $\Pi_{\le k}$ denote the $L^2(\mu)$-orthogonal projection onto $\mathcal F$.
Each $\mathcal H_\ell$ is invariant under $\mathrm{O}(d)$, and hence $\mathcal F$ is closed
under the action of any subgroup $G\subseteq \mathrm{O}(d)$, with invariant subspace
$\mathcal F^G$ of dimension $r_{\mathrm{inv}}$.

Projection estimators for density estimation and regression on
$\mathbb S^{d-1}$ admit kernel representations.
Define the degree-$\ell$ zonal kernel
$
Z_\ell(x,x') := \sum_{j=1}^{N_\ell}\phi_{\ell,j}(x)\phi_{\ell,j}(x'),
$
and the truncated kernel
$
\Pi_{\le k}(x,x') := \sum_{\ell=0}^k Z_\ell(x,x').
$
By the addition theorem for spherical harmonics, $Z_\ell(x,x')$ depends only on
$\langle x,x'\rangle$ and has a closed-form expression in terms of Gegenbauer
polynomials.
Consequently, $\Pi_{\le k}$ can be evaluated using only inner products, with
computational cost $O(k)$ per evaluation via standard three-term recurrences.

As a result, projection estimators take the form
\[
\widehat f(x)
=
\frac{1}{n}\sum_{i=1}^n \Pi_{\le k}(x,x_i)
\quad\text{(density estimation)},
\quad
\widehat f(x)
=
\frac{1}{n}\sum_{i=1}^n y_i\,\Pi_{\le k}(x,x_i)
\quad\text{(regression)}.
\]

For large or continuous groups $G\subseteq  \mathrm{O}(d)$, full data augmentation, averaging over
all group elements, is computationally infeasible.
In contrast, partial data augmentation using a small random subset $S\subseteq G$
is efficient and compatible with the kernel form above.
Our main results show that such partial augmentation suffices to recover the full
statistical benefits of symmetry, with rates governed by $r_{\mathrm{inv}}$ rather
than the size of $G$.

 \section{Conclusion}

In this paper, we studied the statistical and computational role of data
augmentation for learning problems with known symmetries.
Focusing on classical density estimation and regression with finite-dimensional
projection estimators, we developed a general theoretical framework for
understanding when and how \emph{partial} data augmentation can replace full
group-sized augmentation.

Our results show that partial data augmentation is surprisingly powerful.
In expectation, augmenting with a small random subset of group elements is
sufficient to recover the full statistical benefits of symmetry, achieving
minimax-optimal rates identical to those obtained with full augmentation.
Moreover, we proved that a single randomly chosen augmentation set can be reused
uniformly across all functions in the hypothesis space, incurring only a mild
logarithmic overhead.

At the same time, we establish a complementary impossibility result: when the hypothesis space is sufficiently expressive, enforcing \emph{exact} invariance through data augmentation requires averaging over the entire group. This result reveals a sharp separation between approximate and exact symmetry enforcement, and highlights an inherent computational--statistical trade-off. Together, our findings also shed light on recent questions about the statistical role of invariance in learning, and suggest that approximate symmetry may often be the right computational target \citep{diaz2025invariant}.

\section*{Acknowledgments}
BT and MW were partially supported by NSF Award CBET-2112085 and DMS-2406905. MW acknowledges partial funding from an Alfred P. Sloan Fellowship in Mathematics and the AI2050 program at Schmidt Sciences (Grant G-25-69786). SJ acknowledges support from an Alexander von Humboldt Professorship.

   \bibliographystyle{plainnat}
\bibliography{references}

@string{aistats = {Int. Conference on Artificial Intelligence and Statistics (AISTATS)}}

@string{icml = {Int. Conference on Machine Learning (ICML)}}

@string{iclr = {Int. Conference on Learning Representations (ICLR)}}

@string{colt = {Conference on Learning Theory (COLT)}}

@string{neurips = {Advances in Neural Information Processing Systems (NeurIPS)}}

@inproceedings{kaba2023equivariance,
  title={Equivariance with learned canonicalization functions},
  author={Kaba, S{\'e}kou-Oumar and Mondal, Arnab Kumar and Zhang, Yan and Bengio, Yoshua and Ravanbakhsh, Siamak},
  booktitle=icml,
  year={2023}
}

@inproceedings{ma2024canonicalization,
  title={A canonicalization perspective on invariant and equivariant learning},
  author={Ma, George and Wang, Yifei and Lim, Derek and Jegelka, Stefanie and Wang, Yisen},
  booktitle=neurips,
  year={2024}
}

@article{zhang2025artificial,
  title={Artificial intelligence for science in quantum, atomistic, and continuum systems},
  author={Zhang, Xuan and Wang, Limei and Helwig, Jacob and Luo, Youzhi and Fu, Cong and Xie, Yaochen and Liu, Meng and Lin, Yuchao and Xu, Zhao and Yan, Keqiang and others},
  journal={Foundations and Trends{\textregistered} in Machine Learning},
  volume={18},
  number={4},
  pages={385--912},
  year={2025},
  publisher={Now Publishers, Inc.}
}

@inproceedings{tahmasebi2023exact,
  title={The exact sample complexity gain from invariances for kernel regression},
  author={Tahmasebi, Behrooz and Jegelka, Stefanie},
  booktitle=neurips,
  year={2023}
}

@inproceedings{tahmasebigeneralization,
  title={Generalization Bounds for Canonicalization: A Comparative Study with Group Averaging},
  author={Tahmasebi, Behrooz and Jegelka, Stefanie},
  booktitle=iclr,
  year={2025}
}

@inproceedings{puny2021frame,
  title={Frame averaging for invariant and equivariant network design},
  author={Puny, Omri and Atzmon, Matan and Ben-Hamu, Heli and Misra, Ishan and Grover, Aditya and Smith, Edward J and Lipman, Yaron},
  booktitle=iclr,
  year={2022}
}

@inproceedings{dao2019kernel,
  title={A kernel theory of modern data augmentation},
  author={Dao, Tri and Gu, Albert and Ratner, Alexander and Smith, Virginia and De Sa, Chris and R{\'e}, Christopher},
  booktitle=icml,
  year={2019}
}

@article{zhao2022graph,
  title={Graph data augmentation for graph machine learning: A survey},
  author={Zhao, Tong and Jin, Wei and Liu, Yozen and Wang, Yingheng and Liu, Gang and G{\"u}nnemann, Stephan and Shah, Neil and Jiang, Meng},
  journal={arXiv preprint arXiv:2202.08871},
  year={2022}
}

@inproceedings{mikolajczyk2018data,
  title={Data augmentation for improving deep learning in image classification problem},
  author={Miko{\l}ajczyk, Agnieszka and Grochowski, Micha{\l}},
  booktitle={2018 international interdisciplinary PhD workshop (IIPhDW)},
  pages={117--122},
  year={2018},
  organization={IEEE}
}

@article{mumuni2022data,
  title={Data augmentation: A comprehensive survey of modern approaches},
  author={Mumuni, Alhassan and Mumuni, Fuseini},
  journal={Array},
  volume={16},
  pages={100258},
  year={2022},
  publisher={Elsevier}
}

@article{shorten2019survey,
  title={A survey on image data augmentation for deep learning},
  author={Shorten, Connor and Khoshgoftaar, Taghi M},
  journal={Journal of big data},
  volume={6},
  number={1},
  pages={1--48},
  year={2019},
  publisher={Springer}
}

@inproceedings{mei2021learning,
  title={Learning with invariances in random features and kernel models},
  author={Mei, Song and Misiakiewicz, Theodor and Montanari, Andrea},
  booktitle=colt,
  year={2021}
}

@inproceedings{petrache2023approximation,
  title={Approximation-generalization trade-offs under (approximate) group equivariance},
  author={Petrache, Mircea and Trivedi, Shubhendu},
  booktitle=neurips,
  year={2023}
}

@article{v009a015,
 author = {Alon, Noga and Lovett, Shachar},
 title = {Almost $k$-Wise vs. $k$-Wise Independent Permutations, and Uniformity for General Group Actions},
 year = {2013},
 pages = {559--577},
 publisher = {Theory of Computing},
 journal = {Theory of Computing},
 volume = {9},
 number = {15}
}

@article{alon1994random,
  title={Random Cayley graphs and expanders},
  author={Alon, Noga and Roichman, Yuval},
  journal={Random Structures \& Algorithms},
  volume={5},
  number={2},
  pages={271--284},
  year={1994},
  publisher={Wiley Online Library}
}

@inproceedings{shen2022data,
  title={Data augmentation as feature manipulation},
  author={Shen, Ruoqi and Bubeck, S{\'e}bastien and Gunasekar, Suriya},
  booktitle=icml,
  year={2022}
}

@article{lin2024good,
  title={The good, the bad and the ugly sides of data augmentation: An implicit spectral regularization perspective},
  author={Lin, Chi-Heng and Kaushik, Chiraag and Dyer, Eva L and Muthukumar, Vidya},
  journal={Journal of Machine Learning Research},
  volume={25},
  number={91},
  pages={1--85},
  year={2024}
}

@article{chen2020group,
  title={A group-theoretic framework for data augmentation},
  author={Chen, Shuxiao and Dobriban, Edgar and Lee, Jane H},
  journal={Journal of Machine Learning Research},
  volume={21},
  number={245},
  pages={1--71},
  year={2020}
}

@inproceedings{kirichenko2023understanding,
  title={Understanding the detrimental class-level effects of data augmentation},
  author={Kirichenko, Polina and Ibrahim, Mark and Balestriero, Randall and Bouchacourt, Diane and Vedantam, Shanmukha Ramakrishna and Firooz, Hamed and Wilson, Andrew G},
  booktitle=neurips,
  year={2023}
}

@inproceedings{patil2023generalized,
  title={Generalized equivalences between subsampling and ridge regularization},
  author={Patil, Pratik and Du, Jin-Hong},
  booktitle=neurips,
  year={2023}
}

@inproceedings{tahmasebi2025regularity,
  title={Regularity in Canonicalized Models: A Theoretical Perspective},
  author={Tahmasebi, Behrooz and Jegelka, Stefanie},
  booktitle=aistats,
  year={2025}
}

@article{ma2025comprehensive,
  title={A comprehensive survey of data augmentation in visual reinforcement learning},
  author={Ma, Guozheng and Wang, Zhen and Yuan, Zhecheng and Wang, Xueqian and Yuan, Bo and Tao, Dacheng},
  journal={International Journal of Computer Vision},
  pages={1--38},
  year={2025},
  publisher={Springer}
}

@inproceedings{zheng2023toward,
  title={Toward understanding generative data augmentation},
  author={Zheng, Chenyu and Wu, Guoqiang and Li, Chongxuan},
  booktitle=neurips,
  year={2023}
}

@inproceedings{pmlrv202dong23f,
  title = 	 {Adaptively Weighted Data Augmentation Consistency Regularization for Robust Optimization under Concept Shift},
  author =       {Dong, Yijun and Xie, Yuege and Ward, Rachel},
  booktitle = 	icml,
  year = 	 {2023}
}

@article{chen2024comprehensive,
  title={A comprehensive survey for generative data augmentation},
  author={Chen, Yunhao and Yan, Zihui and Zhu, Yunjie},
  journal={Neurocomputing},
  volume={600},
  pages={128167},
  year={2024},
  publisher={Elsevier}
}

@article{li2022data,
  title={Data augmentation approaches in natural language processing: A survey},
  author={Li, Bohan and Hou, Yutai and Che, Wanxiang},
  journal={Ai Open},
  volume={3},
  pages={71--90},
  year={2022},
  publisher={Elsevier}
}

@article{pellicer2023data,
  title={Data augmentation techniques in natural language processing},
  author={Pellicer, Lucas Francisco Amaral Orosco and Ferreira, Taynan Maier and Costa, Anna Helena Reali},
  journal={Applied Soft Computing},
  volume={132},
  pages={109803},
  year={2023},
  publisher={Elsevier}
}

@inproceedings{yang2023sample,
  title={Sample efficiency of data augmentation consistency regularization},
  author={Yang, Shuo and Dong, Yijun and Ward, Rachel and Dhillon, Inderjit S and Sanghavi, Sujay and Lei, Qi},
  booktitle=aistats,
  year={2023}
}

@inproceedings{
soleymani2025from,
title={From Finite to Infinite Groups: A Polynomial-Time Algorithm for Learning with Exact Invariances},
author={Ashkan Soleymani and Behrooz Tahmasebi and Patrick Jaillet and Stefanie Jegelka},
booktitle={NeurIPS 2025 Workshop on Symmetry and Geometry in Neural Representations},
year={2025}
}

@inproceedings{soleymani2025robust,
  title={A Robust Kernel Statistical Test of Invariance: Detecting Subtle Asymmetries},
  author={Soleymani, Ashkan and Tahmasebi, Behrooz and Jegelka, Stefanie and Jaillet, Patrick},
  booktitle=aistats,
  year={2025}
}

@inproceedings{soleymani2025learning,
  title={Learning with Exact Invariances in Polynomial Time},
  author={Soleymani, Ashkan and Tahmasebi, Behrooz and Jegelka, Stefanie and Jaillet, Patrick},
  booktitle=icml,
  year={2025}
}

@article{tropp2012user,
  title={User-friendly tail bounds for sums of random matrices},
  author={Tropp, Joel A},
  journal={Foundations of computational mathematics},
  volume={12},
  number={4},
  pages={389--434},
  year={2012},
  publisher={Springer}
}

@article{bourgain2008uniform,
  title={Uniform expansion bounds for Cayley graphs of $\mathrm{SL}_2(\mathbb{F}_p)$},
  author={Bourgain, Jean and Gamburd, Alex},
  journal={Annals of Mathematics},
  pages={625--642},
  year={2008},
  publisher={JSTOR}
}

@article{batzner20223,
  title={E(3)-equivariant graph neural networks for data-efficient and accurate interatomic potentials},
  author={Batzner, Simon and Musaelian, Albert and Sun, Lixin and Geiger, Mario and Mailoa, Jonathan P and Kornbluth, Mordechai and Molinari, Nicola and Smidt, Tess E and Kozinsky, Boris},
  journal={Nature communications},
  volume={13},
  number={1},
  pages={2453},
  year={2022},
  publisher={Nature Publishing Group UK London}
}

@article{bronstein2017geometric,
  title={Geometric deep learning: going beyond euclidean data},
  author={Bronstein, Michael M and Bruna, Joan and LeCun, Yann and Szlam, Arthur and Vandergheynst, Pierre},
  journal={IEEE Signal Processing Magazine},
  volume={34},
  number={4},
  pages={18--42},
  year={2017},
  publisher={IEEE}
}

@article{smidt2021euclidean,
  title={Euclidean symmetry and equivariance in machine learning},
  author={Smidt, Tess E},
  journal={Trends in Chemistry},
  volume={3},
  number={2},
  pages={82--85},
  year={2021},
  publisher={Elsevier}
}

@inproceedings{dymequivariant,
  title={Equivariant Frames and the Impossibility of Continuous Canonicalization},
  author={Dym, Nadav and Lawrence, Hannah and Siegel, Jonathan W},
  booktitle=icml,
  year={2024}
}

@inproceedings{finzi2021residual,
  title={Residual pathway priors for soft equivariance constraints},
  author={Finzi, Marc and Benton, Gregory and Wilson, Andrew G},
  booktitle=neurips,
  year={2021}
}

@inproceedings{romero2022learning,
  title={Learning partial equivariances from data},
  author={Romero, David W and Lohit, Suhas},
  booktitle=neurips,
  year={2022}
}

@inproceedings{van2022relaxing,
  title={Relaxing equivariance constraints with non-stationary continuous filters},
  author={van der Ouderaa, Tycho and Romero, David W and van der Wilk, Mark},
  booktitle=neurips,
  year={2022}
}

@inproceedings{kim2023regularizing,
  title={Regularizing towards soft equivariance under mixed symmetries},
  author={Kim, Hyunsu and Lee, Hyungi and Yang, Hongseok and Lee, Juho},
  booktitle=icml,
  year={2023}
}

@InProceedings{pmlr-v258-park25d,
  title = 	 {Approximate Equivariance in Reinforcement Learning},
  author =       {Park, Jung Yeon and Bhatt, Sujay and Zeng, Sihan and Wong, Lawson L.S. and Koppel, Alec and Ganesh, Sumitra and Walters, Robin},
  booktitle = 	aistats, 
  year = 	 {2025}
}

@inproceedings{wang2022approximately,
  title={Approximately equivariant networks for imperfectly symmetric dynamics},
  author={Wang, Rui and Walters, Robin and Yu, Rose},
  booktitle=icml,
  year={2022}
}

@inproceedings{yang2024latent,
  title={Latent Space Symmetry Discovery},
  author={Yang, Jianke and Dehmamy, Nima and Walters, Robin and Yu, Rose},
  booktitle=icml,
  year={2024}
}

@inproceedings{yang2023generative,
  title={Generative adversarial symmetry discovery},
  author={Yang, Jianke and Walters, Robin and Dehmamy, Nima and Yu, Rose},
  booktitle=icml,
  year={2023}
}

@inproceedings{van2023learning,
  title={Learning layer-wise equivariances automatically using gradients},
  author={van der Ouderaa, Tycho and Immer, Alexander and van der Wilk, Mark},
  booktitle=neurips,
  year={2023}
}

@inproceedings{
huh2025a,
title={Discovering Group Structures via Unitary Representation Learning},
author={Dongsung Huh},
booktitle=iclr,
year={2025}
}

@article{desai2022symmetry,
  title={Symmetry discovery with deep learning},
  author={Desai, Krish and Nachman, Benjamin and Thaler, Jesse},
  journal={Physical Review D},
  volume={105},
  number={9},
  pages={096031},
  year={2022},
  publisher={APS}
}

@inproceedings{shaw2024symmetry,
  title={Symmetry discovery beyond affine transformations},
  author={Shaw, Ben and Magner, Abram and Moon, Kevin},
  booktitle=neurips,
  year={2024}
}

@inproceedings{shumaylov2025lie,
  title={Lie algebra canonicalization: Equivariant neural operators under arbitrary lie groups},
  author={Shumaylov, Zakhar and Zaika, Peter and Rowbottom, James and Sherry, Ferdia and Weber, Melanie and Sch{\"o}nlieb, Carola-Bibiane},
  booktitle=iclr,
  year={2025}
}

@article{batzner2023advancing,
  title={Advancing molecular simulation with equivariant interatomic potentials},
  author={Batzner, Simon and Musaelian, Albert and Kozinsky, Boris},
  journal={Nature Reviews Physics},
  volume={5},
  number={8},
  pages={437--438},
  year={2023},
  publisher={Nature Publishing Group UK London}
}

@book{dai2013approximation,
  title={Approximation theory and harmonic analysis on spheres and balls},
  author={Dai, Feng},
  year={2013},
  publisher={Springer}
}

@article{diaz2025invariant,
  title={Invariant kernels: Rank stabilization and generalization across dimensions},
  author={D{\'\i}az, Mateo and Drusvyatskiy, Dmitriy and Kendrick, Jack and Thomas, Rekha R},
  journal={arXiv preprint arXiv:2502.01886},
  year={2025}
}

@article{weber2025geometric,
author = {Weber, Melanie},
title = {Geometric Machine Learning},
journal = {AI Magazine},
volume = {46},
number = {1},
pages = {e12210},
year = {2025}
}

@inproceedings{kiani2024hardness,
  title={On the hardness of learning under symmetries},
  author={Kiani, Bobak and Le, Thien and Lawrence, Hannah and Jegelka, Stefanie and Weber, Melanie},
  booktitle={International Conference on Learning Representations},
  volume={2024},
  pages={25956--26003},
  year={2024}
}

@article{pacini2025universality,
  title={On universality of deep equivariant networks},
  author={Pacini, Marco and Petrache, Mircea and Lepri, Bruno and Trivedi, Shubhendu and Walters, Robin},
  journal={arXiv preprint arXiv:2510.15814},
  year={2025}
}

@inproceedings{tahmasebi2025achieving,
  title={Achieving Approximate Symmetry Is Exponentially Easier than Exact Symmetry},
  author={Tahmasebi, Behrooz and Weber, Melanie},
  booktitle=iclr,
  year={2026}
}

@inproceedings{bietti2021sample,
  title={On the sample complexity of learning under geometric stability},
  author={Bietti, Alberto and Venturi, Luca and Bruna, Joan},
  booktitle= neurips,
  year={2021}
}

@inproceedings{tahmasebi2026adaptive,
title = {Adaptive Symmetry Discovery for Dynamical System Identification},
author = {Tahmasebi, Behrooz and Weber, Melanie},
booktitle = icml,
year = {2026}
}

@inproceedings{dehmamy2021automatic,
  title={Automatic symmetry discovery with lie algebra convolutional network},
  author={Dehmamy, Nima and Walters, Robin and Liu, Yanchen and Wang, Dashun and Yu, Rose},
  booktitle=neurips,
  year={2021}
}

 \clearpage
\appendix

\section{Preliminaries}\label{apd:bck}

\subsection{Spherical Harmonics}\label{apd:bck:sh}

We briefly review basic facts about spherical harmonics on the unit sphere \citep{dai2013approximation}.
Let
\[
\mathbb S^{d-1} := \left\{ x \in \mathbb R^d : \|x\|_2 = 1 \right\}
\]
be the unit sphere in $\mathbb R^d$, equipped with the normalized surface measure $\mu$,
so that $\int_{\mathbb S^{d-1}} d\mu(x) = 1$.
For functions $f,h : \mathbb S^{d-1} \to \mathbb R$, we define the $L^2$ inner product
\[
\langle f,h\rangle_{L^2(\mathbb S^{d-1})}
:= \int_{\mathbb S^{d-1}} f(x)\, h(x)\, d\mu(x),
\]
and denote by $L^2(\mathbb S^{d-1})$ the associated Hilbert space.

The (Euclidean) Laplacian on $\mathbb R^d$ is defined by
\[
\Delta := \sum_{i=1}^d \frac{\partial^2}{\partial x_i^2}.
\]
In contrast, the Laplace--Beltrami operator $\Delta_{\mathbb S^{d-1}}$ is the intrinsic
Laplacian associated with the Riemannian metric on the sphere.
While $\Delta$ acts on functions defined in $\mathbb R^d$, $\Delta_{\mathbb S^{d-1}}$
acts on functions defined on $\mathbb S^{d-1}$ and can be viewed as the restriction of
$\Delta$ to tangential directions along the sphere.

Let $-\Delta_{\mathbb S^{d-1}}$ denote the (positive semidefinite) Laplace--Beltrami
operator on $\mathbb S^{d-1}$.
Its spectrum is discrete and indexed by $\ell = 0,1,2,\ldots$, with eigenvalues
\[
\lambda_\ell := \ell(\ell + d - 2).
\]

\paragraph{Spherical harmonics.}
The eigenspace corresponding to $\lambda_\ell$ is denoted by $\mathcal H_\ell$ and consists
of the restrictions to $\mathbb S^{d-1}$ of homogeneous harmonic polynomials of degree $\ell$
in $\mathbb R^d$, namely
\[
\mathcal H_\ell
:= \Big\{ p|_{\mathbb S^{d-1}} :
p \text{ is homogeneous of degree } \ell
\text{ and } \Delta p = 0 \text{ in } \mathbb R^d \Big\}.
\]
The spaces $\{\mathcal H_\ell\}_{\ell \ge 0}$ are mutually orthogonal in
$L^2(\mathbb S^{d-1})$ and yield the orthogonal decomposition
\[
L^2(\mathbb S^{d-1})
= \bigoplus_{\ell=0}^\infty \mathcal H_\ell .
\]
Let us denote the multiplicity of the eigenvalue $\lambda_\ell$ as
\[
N_\ell := \dim(\mathcal H_\ell)
= \binom{d+\ell-1}{\ell} - \binom{d+\ell-3}{\ell-2}
\]
For each $\ell$, let $\{\phi_{\ell,m}\}_{m=1}^{N_\ell}$ be an orthonormal basis of
$\mathcal H_\ell$.

\paragraph{Asymptotics of the multiplicities.}
For fixed degree $\ell$ and dimension $d \to \infty$, the multiplicity $N_\ell$ satisfies
\[
N_\ell
=  \frac{d^\ell}{\ell!} +  O_{\ell}\!\left(d^{\ell-1}\right).
\]
where the implicit constant depends only on $\ell$.
In particular, $N_\ell$ grows polynomially in $d$ with leading order $d^{\ell}$. Moreover, the dimension of the space of spherical harmonics of degree at most $k$ satisfies
\[
 \sum_{\ell=0}^k N_\ell = \binom{d+k-1}{k} + \binom{d+k-2}{k-1},
\]
In particular, for fixed $k$ and $d\to\infty$,
\[ \sum_{\ell=0}^k N_\ell =  \frac{d^k}{k!} +O_{k}\!\left(d^{k-1}\right),
\]
where the implicit constant depends only on $k$.

\paragraph{Harmonic expansion.}
Every function $f \in L^2(\mathbb S^{d-1})$ admits the convergent expansion
\[
f
= \sum_{\ell=0}^\infty \sum_{m=1}^{N_\ell}
f_{\ell,m}\,\phi_{\ell,m},
\qquad
f_{\ell,m}
:= \langle f,\phi_{\ell,m}\rangle_{L^2(\mathbb S^{d-1})}.
\]
Truncating the expansion at degrees $\ell \le k$ yields the $L^2$-orthogonal projection
of $f$ onto the space of spherical polynomials of degree at most $k$, and provides the
best $L^2$ approximation of $f$ within this class.

\subsection{Gegenbauer Polynomials, Zonal Kernels, and Projection Kernels}\label{apd:bck:sh2}

A \emph{zonal}  function $Z:\mathbb S^{d-1}\times \mathbb S^{d-1}\to\mathbb R$ is a function that only depends on the inner product, i.e.,
\[
Z(x,y)=\zeta(\langle x,y\rangle), \qquad x,y\in \mathbb S^{d-1},
\]
for some $\zeta:[-1,1]\to\mathbb R$.

\paragraph{Gegenbauer polynomials.}
Let $\lambda := \frac{d-2}{2}$. The Gegenbauer polynomials
$\{C_\ell^{(\lambda)}\}_{\ell\ge 0}$ form an orthogonal family on $[-1,1]$ with respect
to the weight $(1-t^2)^{\lambda-\frac12}$.  

\paragraph{Zonal harmonics (reproducing kernels).}
For each $\ell\ge 0$, let $\mathcal H_\ell$ be the eigenspace of the Laplace--Beltrami
operator $-\Delta_{\mathbb S^{d-1}}$ with eigenvalue $\lambda_\ell=\ell(\ell+d-2)$, and
let $\{\phi_{\ell,m}\}_{m=1}^{N_\ell}$ be an orthonormal basis of $\mathcal H_\ell$ in
$L^2(\mathbb S^{d-1})$. Define the \emph{zonal harmonic} (also called the reproducing
kernel of $\mathcal H_\ell$) by
\[
Z_\ell(x,y) := \sum_{m=1}^{N_\ell} \phi_{\ell,m}(x)\phi_{\ell,m}(y).
\]
Note that $Z_\ell(x,y)$ is a function of
$\langle x,y\rangle$ and can be written explicitly as
\[
Z_\ell(x,y) = \frac{\ell+\lambda}{\lambda}\, C_\ell^{(\lambda)}(\langle x,y\rangle),
\qquad \lambda=\frac{d-2}{2}.
\]
In particular, $Z_\ell(x,x)=N_\ell$ is constant in $x$.

\paragraph{Projection kernel to degree $\le k$.}
Let $\Pi_{\le k}$ denote the $L^2(\mathbb S^{d-1})$-orthogonal projector onto the space
$\bigoplus_{\ell=0}^k \mathcal H_\ell$ of spherical polynomials of degree at most $k$.
Then $\Pi_{\le k}$ is an integral operator with kernel
\[
\Pi_{\le k}(x,y) := \sum_{\ell=0}^k Z_\ell(x,y)
= \sum_{\ell=0}^k \frac{\ell+\lambda}{\lambda}\, C_\ell^{(\lambda)}(\langle x,y\rangle),
\]
so that
\[
(\Pi_{\le k} f)(x) = \int_{\mathbb S^{d-1}} \Pi_{\le k}(x,y)\, f(y)\, d\sigma(y).
\]
Hence, the projection depends only on inner products and can be evaluated efficiently:
for fixed $x$, computing $\Pi_{\le k}(x,y)$ reduces to evaluating
$C_\ell^{(\lambda)}(t)$ for $t=\langle x,y\rangle$ and $\ell=0,\dots,k$, which can be
done in $O(k)$ time using Gegenbauer polynomials.

\subsection{Groups}
\label{app:groups}

A \emph{group} is a pair $(G,\cdot)$ where $G$ is a set and $\cdot: G\times G\to G$
is a binary operation (often written as multiplication) such that:
\begin{itemize}
  \item \textbf{(Associativity)} $(g\cdot h)\cdot k = g\cdot (h\cdot k)$ for all $g,h,k\in G$.
  \item \textbf{(Identity)} There exists an element $e\in G$ such that $e\cdot g = g\cdot e = g$
  for all $g\in G$.
  \item \textbf{(Inverse)} For every $g\in G$ there exists $g^{-1}\in G$ such that
  $g\cdot g^{-1} = g^{-1}\cdot g = e$.
\end{itemize}
Throughout the paper, we omit the ``$\cdot$'' notation and write $gh$ for the group product for all $g,h\in G$.

Here we list a number of example groups.

\begin{itemize}
  \item \textbf{Permutation group.}
  The symmetric group $S_d$ is the set of all bijections $\pi:\{1,\dots,d\}\to\{1,\dots,d\}$
  with group operation given by composition. Equivalently, each $\pi\in S_d$ is a permutation
  of the indices $\{1,\dots,d\}$.

   \item \textbf{Sign-flipping group.}
The sign-flipping group consists of all vectors $\varepsilon=(\varepsilon_1,\dots,\varepsilon_d)$
with $\varepsilon_i\in\{+1,-1\}$, equipped with componentwise multiplication.

  \item \textbf{Hyperoctahedral group (signed permutations).}
  The hyperoctahedral group $B_d$ is the group of all \emph{signed permutation matrices}
  in $\mathbb{R}^{d\times d}$, i.e., matrices $Q$ such that each row and each column
  contains exactly one nonzero entry, and every nonzero entry equals $+1$ or $-1$.
  Equivalently, each element of $B_d$ can be specified by a pair $(\pi,\varepsilon)$ where
  $\pi\in S_d$ and $\varepsilon\in\{\pm 1\}^d$.
  This group models invariance under the reordering of coordinates, together with independent
  sign flips.

  \item \textbf{Cyclic group.}
  The cyclic group of order $d$, denoted $C_d$, is generated by a single element $r$ with
  relation $r^d=e$:
  \[
  C_d := \{e, r, r^2,\dots,r^{d-1}\}, \qquad r^a r^b = r^{a+b \ (\mathrm{mod}\ d)}.
  \]
  A canonical example is the group of circular shifts on $d$ coordinates.

  \item \textbf{Dihedral group.}
  The dihedral group of order $2d$, denoted $D_d$, is the symmetry group of a regular $d$-gon.
  It contains $d$ rotations and $d$ reflections, and can be presented by generators
  $r$ (a rotation by $2\pi/d$) and $s$ (a reflection) satisfying
  \[
  r^d = e, \qquad s^2 = e, \qquad srs = r^{-1}.
  \]
  Each element of $D_d$ can be written uniquely as either $r^k$ or $sr^k$ for some
  $k\in\{0,1,\dots,d-1\}$. In applications on $d$-tuples, one may realize $r$ as a cyclic
  shift and $s$ as reversal (a ``flip''), generating both shifts and flip-symmetries.
\end{itemize}

We also frequently encounter infinite groups:
\begin{itemize}
  \item \textbf{Orthogonal group.} The orthogonal group $\mathrm{O}(d)$ is
  \[
  \mathrm{O}(d) := \{Q\in\mathbb{R}^{d\times d} : Q^\top Q = I_d\},
  \]
  i.e., the set of linear maps of $\mathbb{R}^d$ that preserve Euclidean inner products
  (and hence distances). These include rotations and reflections.

  \item \textbf{Special orthogonal group.} The special orthogonal group $\mathrm{SO}(d)$ is the subgroup
  \[
  \mathrm{SO}(d) := \{Q\in O(d) : \det(Q)=1\},
  \]
  consisting of proper rotations (orientation-preserving orthogonal transformations).

    \item \textbf{Unitary group.}
The unitary group $\mathrm{U}(d)$ is
\[
\mathrm{U}(d) := \{U\in\mathbb{C}^{d\times d} : U^* U = I_d\},
\]
where $U^*$ denotes the conjugate transpose of $U$. Equivalently, $\mathrm{U}(d)$ consists of
all linear maps of $\mathbb{C}^d$ that preserve the complex inner product.

  \item \textbf{Translation group.} The translation group $(\mathbb{R}^d,+)$ acts on
  $\mathbb{R}^d$ by $t\cdot x = x+t$. This models translation invariance in Euclidean spaces.
\end{itemize}

For compact groups $G$, there exists a unique probability measure on $G$,
called the \emph{Haar measure}, which is invariant under left group translations.
Sampling $g\in G$ therefore corresponds to sampling a group element uniformly at
random throughout the paper. We write $\mathbb{E}_g[\cdot]$ to denote expectation with respect to uniform sampling
from the group.

\subsection{Group Actions}

Let $(\mathcal X,\mu)$ be a measured space. A (left) \emph{group action} of a group $G$ on
$\mathcal X$ is a map $(g,x)\mapsto g x$ satisfying
\[
    e x = x
    \quad\text{and}\quad
    (g_1 g_2)x = g_1(g_2 x),
    \qquad
    \forall g_1,g_2\in G,\ \forall x\in\mathcal X .
\]
We assume that the action is \emph{isometric}, i.e.,
$\mu(gA)=\mu(A)$ for all measurable $A\subseteq\mathcal X$ and $g\in G$.

\paragraph{Lifted action on function spaces.}
Let $\mathcal{F} \subset L^2(\mathcal{X},\mu)$ be a finite-dimensional vector space of
continuous functions. Let $\{\phi_\ell\}_{\ell=1}^r$ be a fixed orthonormal basis of
$\mathcal{F}$,  where
$r \coloneqq \dim(\mathcal{F})$. Therefore, we may identify $\mathcal{F}$ with $\mathbb{R}^r$. With a slight abuse of
notation, each function $f\in\mathcal{F}$ is identified with a vector $f\in\mathbb{R}^r$.

We say that $\mathcal{F}$ is \emph{closed} under the group action if, for any $f\in\mathcal{F}$,
we have $x \mapsto f(gx) \in \mathcal{F}$ for all $g\in G$. Throughout, we assume that $\mathcal{F}$ is closed under the group action.

The action of $G$ on $\mathcal X$ induces
a natural action on $\mathcal F$ defined by
\[
(T_g f)(x) := f(g^{-1} x),
\qquad \forall g\in G,
\]
where $T_g : \mathcal{F} \to \mathcal{F}$ are unitary operators (matrices), due to the isometry
assumption. When $\mathcal{F}$ is finite-dimensional, we can identify each $T_g$ with a matrix
of dimension $\dim(\mathcal{F})$, and thus $T_g \in \mathbb{R}^{r\times r}$ for all $g\in G$,
where $r \coloneqq \dim(\mathcal{F})$. In particular, $T_e = I_r$, where $e\in G$ denotes the
identity element and $I_r \in \mathbb{R}^{r\times r}$ is the identity matrix, and we have $T_gT_h = T_{gh}$ for all $g,h\in G$.

\paragraph{Invariant functions and projection.}
The $G$-invariant subspace of $\mathcal F$ is
\[
\mathcal F^G := \{ f\in \mathcal F : T_g f = f \;\; \forall g\in G\}
= \bigcap_{g\in G} \ker(T_g - I).
\]
We identify $\mathcal{F}^G \subseteq \mathcal{F}$ with a finite-dimensional vector subspace
of $\mathbb{R}^r$, where $r_{\operatorname{inv}} \coloneqq \dim(\mathcal{F}^G)$.
For compact groups $G$, the orthogonal projection onto the invariant subspace
$\mathcal{F}^G$ is given by
\[
\Pi_G f \;\coloneqq\; \mathbb{E}_{g}\!\left[\,T_g f\,\right].
\]

\subsection{Group Representations}\label{app:group_rep}

A (unitary) representation of a group $G$ on a finite-dimensional Hilbert space $\mathcal{H}$ is a map
\[
\rho : G \to \mathcal{U}(\mathcal{H}),
\]
from $G$ to the group of unitary operators on $\mathcal{H}$, such that
\[
\rho(e) = I_{\mathcal{H}}
\qquad\text{and}\qquad
\rho(gh) = \rho(g)\rho(h)
\quad \text{for all } g,h\in G.
\]
Indeed, the group multiplication in $G$ is represented by the composition of
unitary operators on $\mathcal{H}$.

A subspace $W \subseteq \mathcal{H}$ is said to be \emph{$G$-invariant} if
\[
\rho(g) w \in W
\quad \text{for all } w\in W \text{ and all } g\in G,
\]
that is, the action of the group does not move vectors in $W$ outside of $W$.

A representation $\rho$ is called \emph{irreducible} if the only closed
$G$-invariant subspaces of $\mathcal{H}$ are the trivial ones, namely
$\{0\}$ and $\mathcal{H}$ itself. Otherwise, the representation is said to be
\emph{reducible}.

\paragraph{Finite groups and harmonic decompositions.}
For finite groups, every finite-dimensional unitary representation decomposes
orthogonally into irreducible components:
\[
\mathcal{H} \;\cong\; \bigoplus_{\lambda\in\widehat G}
\mathbb{C}^{m_\lambda}\otimes V_\lambda,
\]
where $\widehat G$ denotes the set of inequivalent irreducible representations,
$V_\lambda$ is an irreducible representation of dimension $d_\lambda \in \mathbb{N}$,
and $m_\lambda \in \mathbb{Z}_{\ge 0}$ is its multiplicity.
The number of irreducible representations equals the number of conjugacy classes
of $G$, and their dimensions satisfy
\[
\sum_{\lambda \in \widehat{G}} d_\lambda^2 = |G|.
\]
In particular, $|\widehat{G}| \le |G|$ for all finite groups $G$.

This decomposition is the finite-group analogue of the harmonic decompositions
appearing in classical settings (e.g., spherical harmonics), and induces an
orthogonal decomposition of $\mathcal{F}$ into \emph{isotypic components}.

\paragraph{Matrix structure under a change of coordinates.}
Let $\rho : G \to \mathcal{U}(\mathcal{H})$ be a finite-dimensional unitary
representation. There exists an orthonormal basis of $\mathcal{H}$ under which
each operator $\rho(g)$ takes a block-diagonal form consistent with the above
decomposition:
\[
\rho(g)
\;\cong\;
\bigoplus_{\lambda\in\widehat G}
I_{m_\lambda} \otimes \rho_\lambda(g),
\qquad g\in G,
\]
where $\rho_\lambda : G \to \mathcal{U}(V_\lambda)$ denotes an irreducible
representation of dimension $d_\lambda$, and $I_{m_\lambda}$ is the identity
operator on the multiplicity space $\mathbb{C}^{m_\lambda}$.

In this basis, each irreducible representation appears $m_\lambda$ times, and the
action of $G$ is identical across these copies. This block structure plays a
central role in characterizing invariant subspaces and projection operators in
finite-dimensional models.

\paragraph{Invariant subspaces in finite-dimensional models.}

Specializing the above discussion to $\mathcal{H}=\mathcal{F}$, we note that each
operator $T_g$ is unitary on $\mathcal{F}$, and the map $g\mapsto T_g$ defines a
finite-dimensional unitary representation
$
\rho : G \to U(r)
$
on $\mathcal{F}$ (or, equivalently, on $\mathbb{R}^r$).

The \emph{trivial irreducible representation} of $G$ is the one-dimensional
representation in which every group element acts as the identity, i.e.,
$\rho_{\mathrm{triv}}(g)=1 \in \mathbb{R}$ for all $g\in G$.
Functions lying in the trivial isotypic component are exactly those that are
fixed by the action of $G$, and therefore correspond precisely to
$G$-invariant functions $\mathcal{F}^G$, and can be obtained via orthogonal
projection.

\subsection{Projection Estimators}
\label{app:projection-estimators}

In this subsection, we review projection-based estimators for density estimation
and regression in a general finite-dimensional Hilbert space setting. These estimators form the
conceptual basis for the spherical harmonics constructions discussed later.

To review, let $(\mathcal{X},\mu)$ be a measured space and let
$\mathcal{F} \subset L^2(\mathcal{X},\mu)$ be a finite-dimensional linear subspace
of dimension $r$.
Note that every function $f\in \mathcal{F}$ admits the expansion
\[
f = \sum_{\ell=1}^r \theta_\ell \phi_\ell,
\qquad
\theta_\ell \coloneqq \langle f, \phi_\ell \rangle.
\]

\paragraph{Projection density estimation (unlabeled data).}
Let $x_1,\dots,x_n$ be i.i.d.\ samples drawn from an unknown density
$f^\star \in L^2(\mathcal X,\mu)$.
The $L^2(\mathcal{X})$-orthogonal projection of $f^\star$ onto $\mathcal{F}$ is
\[
\Pi_{\mathcal{F}} f^\star
:= \arg\min_{f\in\mathcal{F}} \|f^\star-f\|_{L^2(\mathcal{X})}^2
= \sum_{\ell=1}^r \theta_\ell \phi_\ell,
\qquad
\theta_\ell = \mathbb{E}_{x\sim f^\star}[\phi_\ell(x)].
\]
Since the coefficients $\theta_\ell$ are expectations, they admit unbiased
empirical estimators
\[
\widehat{\theta}_\ell := \frac{1}{n}\sum_{i=1}^n \phi_\ell(x_i).
\]
This leads to the \emph{projection density estimator}
\[
\widehat{f} := \sum_{\ell=1}^r \widehat{\theta}_\ell \phi_\ell.
\]

\paragraph{Projection regression estimation (labeled data).}
Now consider a supervised setting where $(x_i,y_i)_{i=1}^n$ are i.i.d.\ samples
from a joint distribution on $\mathcal X\times\mathbb R$, where
$x_i \sim \mu$ is drawn uniformly from $\mathcal X$, and
\[
y_i = f^\star(x_i) + \varepsilon_i,
\qquad
\mathbb{E}[\varepsilon_i]=0,
\qquad
\mathbb{E}[\varepsilon_i^2]=\sigma^2,
\]
with $\varepsilon_1,\dots,\varepsilon_n$ independent of
$x_1,\dots,x_n$.

The regression function $f^\star(x)=\mathbb{E}[y |x]$ lies in
$L^2(\mathcal X,\mu)$, and its $L^2(\mathcal{X})$-orthogonal projection onto $\mathcal{F}$ is given by
\[
\Pi_{\mathcal{F}} f^\star
:= \arg\min_{f\in\mathcal{F}} \mathbb{E}\big[(f(x)-f^\star(x))^2\big]
= \sum_{\ell=1}^r \beta_\ell \phi_\ell,
\qquad
\beta_\ell = \mathbb{E}_{x,y}[y\,\phi_\ell(x)].
\]
The coefficients $\beta_\ell$ admit empirical estimators
\[
\widehat{\beta}_\ell := \frac{1}{n}\sum_{i=1}^n y_i\,\phi_\ell(x_i),
\]
leading to the \emph{projection regression estimator}
\[
\widehat{f} := \sum_{\ell=1}^r \widehat{\beta}_\ell \phi_\ell.
\]

\paragraph{Relation to the method of moments.}
Both the projection density estimator and the projection regression estimator can
be interpreted as instances of the \emph{method of moments}. In the density
estimation setting, the population projection coefficients satisfy
\[
\theta_\ell \;=\; \mathbb{E}\big[\phi_\ell(x)\big],
\qquad \ell\in[r],
\]
and the empirical coefficients
$\widehat{\theta}_\ell = \frac{1}{n}\sum_{i=1}^n \phi_\ell(x_i)$ are obtained by
matching these moments. The resulting estimator
$\widehat{f}=\sum_{\ell=1}^r \widehat{\theta}_\ell\phi_\ell$ is therefore a
moment-based approximation of the $L^2(\mathcal{X})$-projection of $f^\star$ onto
$\mathcal{F}$.

Similarly, in the regression setting, the population projection coefficients
satisfy the moment identities
\[
\beta_\ell \;=\; \mathbb{E}\big[y\,\phi_\ell(x)\big],
\qquad \ell\in[r],
\]
and are estimated by their empirical counterparts
$\widehat{\beta}_\ell = \frac{1}{n}\sum_{i=1}^n y_i\,\phi_\ell(x_i)$. In both cases,
the estimators are obtained by matching population moments with empirical moments
and substituting the resulting coefficients into the basis expansion.

\subsection{Projection Estimators on  Sphere}
\label{apd:bck:sh3}

We now specialize the projection estimators to the case
where the input space is the unit sphere and the approximation space is chosen
according to spherical harmonics. This specialization yields classical spectral
estimators and admits efficient kernel representations via zonal harmonics.

\paragraph{Spherical harmonics and approximation spaces.}
Let $\mathcal X = \mathbb{S}^{d-1}$ be the unit sphere equipped with the uniform
probability measure $\mu$.
For each $\ell\ge 0$, let $\mathcal H_\ell$ be the eigenspace of spherical harmonics with degree $\ell$, and
let $\{\phi_{\ell,m}\}_{m=1}^{N_\ell}$ be an orthonormal basis of $\mathcal H_\ell$ in
$L^2(\mathbb S^{d-1})$, where $N_\ell$ is the dimension of this space.

For a fixed truncation level $k \in \mathbb{Z}_{\ge 0}$, we define the finite-dimensional
approximation space
\[
\mathcal{F}
\;:=\;
\bigoplus_{\ell=0}^k \mathcal{H}_\ell, \qquad  r \coloneqq \dim(\mathcal{F})=\; \sum_{\ell=0}^k N_\ell.
\] 
Note that, in contrast to the general setting of the previous subsection, where basis
functions were indexed by a single index $\ell\in[r]$, here each basis function
is indexed by a degree $\ell$ and a multiplicity index $m$.

\paragraph{Projection estimators in spherical harmonics coordinates.}
Any function $f\in\mathcal{F}$ admits the expansion
\[
f(x)
\;=\;
\sum_{\ell=0}^k \sum_{m=1}^{N_\ell} \theta_{\ell,m}\,\phi_{\ell,m}(x),
\qquad
\theta_{\ell,m} = \langle f, \phi_{\ell,m} \rangle.
\]
In the density estimation setting, where $x_1,\dots,x_n$ are i.i.d.\ samples from
an unknown density $f^\star$ on $\mathbb{S}^{d-1}$, the population projection
coefficients satisfy
\[
\theta_{\ell,m}
\;=\;
\mathbb{E}_{x\sim f^\star}\big[\phi_{\ell,m}(x)\big] \implies \widehat{\theta}_{\ell,m}
\;=\;
\frac{1}{n}\sum_{i=1}^n \phi_{\ell,m}(x_i).
\]
The resulting projection estimator is
\[
\widehat{f}(x)
\;=\;
\sum_{\ell=0}^k \sum_{m=1}^{N_\ell} \widehat{\theta}_{\ell,m}\,\phi_{\ell,m}(x).
\]

An analogous expression holds in the regression setting, with coefficients
$\theta_{\ell,m}$ replaced by $\beta_{\ell,m}=\mathbb{E}[y\,\phi_{\ell,m}(x)]$.

\paragraph{Projection kernels and zonal harmonics.}
Rather than working explicitly with the basis $\{\phi_{\ell,m}\}$, it is often
convenient to express the projection estimator using the associated
\emph{projection kernel}
\[
\Pi_{\le k}(x,x')
\;:=\;
\sum_{\ell=0}^k \sum_{m=1}^{N_\ell} \phi_{\ell,m}(x)\,\phi_{\ell,m}(x').
\]
With this notation, the projection estimator admits the kernel form
\[
\widehat{f}(x)
\;=\;
\frac{1}{n}\sum_{i=1}^n \Pi_{\le k}(x,x_i)
\quad\text{(density estimation)},
\quad
\widehat{f}(x)
\;=\;
\frac{1}{n}\sum_{i=1}^n y_i\,\Pi_{\le k}(x,x_i)
\quad\text{(regression)}.
\]

The kernel $\Pi_{\le k}(x,x')$ is a \emph{zonal kernel}, meaning that it depends on $x$ and
$x'$ only through their inner product $\langle x,x'\rangle$.

\paragraph{Gegenbauer polynomial representation.}
By the addition theorem for spherical harmonics \citep{dai2013approximation}, the projection kernel admits the
explicit representation
\[
\Pi_k(x,x')
\;=\;
 \sum_{\ell=0}^k \frac{\ell+\lambda}{\lambda}\, C_\ell^{(\lambda)}(\langle x,x'\rangle),
\qquad
\lambda = \frac{d-2}{2},
\]
where $C_\ell^{(\lambda)}$ denotes the Gegenbauer polynomial of degree $\ell$.
As a result, both density and regression projection estimators can be evaluated
efficiently using only inner products between data points, without explicitly
computing the spherical harmonics basis functions.

Indeed, Gegenbauer polynomials admit a \emph{three-term recurrence}, meaning that each
polynomial $C_{\ell+1}^{(\lambda)}(t)$ can be computed using only the two preceding
values $C_{\ell}^{(\lambda)}(t)$ and $C_{\ell-1}^{(\lambda)}(t)$.
Consequently, for a fixed $t\in[-1,1]$, the sequence
$\{C_\ell^{(\lambda)}(t)\}_{\ell=0}^k$ can be evaluated iteratively, storing only
two intermediate values at any time. This yields an $O(k)$ time complexity and
$O(1)$ memory usage for computing all degrees up to $k$.

We conclude with an important implication. Although projection estimators are
computationally tractable for polynomial features of moderate degree $k$ growing
polynomially with the dimension $d$, exact symmetry through data augmentation
may still require augmentation over the full group. In other words, averaging
or data augmentation does not, in general, alleviate the computational cost of
projecting onto invariant subspaces, even for moderate-degree polynomial
features. The next proposition shows explicitly that the condition in
Theorem~\ref{thm:exact-invariance-forces-full-group} is already satisfied when
$k=O(d^2)$.

\begin{proposition}[Permutation irreps in low-degree harmonics]
Let $S_d$ act on $\mathbb S^{d-1}$ by permuting coordinates. Then every
irreducible representation of $S_d$ appears in the restriction of spherical
harmonics of degree at most $\binom d2=O(d^2)$. Equivalently, the space of
degree-$O(d^2)$ spherical harmonics already contains all symmetry modes of the
permutation action. See \citep{tahmasebi2025achieving} for further discussion.
\end{proposition}

\subsection{Approximate Projection via Random Group Averaging}
\label{app:approx-projection-random-averaging}

In this subsection, we study approximating the projection onto the invariant
subspace $\mathcal{F}^G$ by averaging over a \emph{random subset} of group elements,
rather than the entire group. This viewpoint is central to understanding why
approximate symmetry can be achieved efficiently.

Recall that $\mathcal{F}\subset L^2(\mathcal X,\mu)$ is a finite-dimensional
Hilbert space of dimension $r$, closed under the group action, and that the lifted
action induces a unitary representation
$
\rho : G \to U(r)$. 
The orthogonal projection onto the invariant subspace $\mathcal{F}^G$ is given by
group averaging,
$
\Pi_G f \;=\; \mathbb{E}_{g}[T_g f].
$

Let $S=\{g_1,\dots,g_s\}$ be a multiset of elements drawn independently and
uniformly from $G$. We define the empirical averaging operator
\[
\Pi_S f
\;\coloneqq\;
\frac{1}{s}\sum_{i=1}^s T_{g_i} f.
\]
This operator can be viewed as an approximation of $\Pi_G$.

\paragraph{Decomposition into irreducible components.}
By finite-group representation theory, there exists an orthonormal change of
coordinates under which $\mathcal{F}$ decomposes as
\[
\mathcal{F}
\;\cong\;
\bigoplus_{\lambda\in\widehat G}
\mathbb{C}^{m_\lambda}\otimes V_\lambda \implies
\rho(g)
\;\cong\;
\bigoplus_{\lambda\in\widehat G}
I_{m_\lambda}\otimes \rho_\lambda(g),
\qquad g\in G.
\]
Note that  
\[
\sum_{\lambda \in \widehat{G}} m_\lambda = r = \dim(\mathcal{F}).
\]
Under this decomposition, the exact projection $\Pi_G$ acts as the identity on
the trivial representation and annihilates all nontrivial irreducible components.
In contrast, the empirical operator $\Pi_S$ replaces the group average on each
block by a finite sample average.

\paragraph{Behavior of random averaging.}
For the trivial irreducible representation, $\rho_{\mathrm{triv}}(g)=1$ for all
$g\in G$, and therefore
\[
\Pi_S f = f
\quad \text{on}\quad \mathcal{F}^G
\quad\text{for any } S.
\]
Thus, invariant components are preserved exactly, regardless of the choice of $S$.

For any nontrivial irreducible representation $\lambda\neq \mathrm{triv}$, the
group average of the representation matrices vanishes:
\[
\mathbb{E}_{g}\big[\rho_\lambda(g)\big] \;=\; 0.
\]
One way to see this is that $A_\lambda := \mathbb{E}_{g}[\rho_\lambda(g)]$
commutes with $\rho_\lambda(h)$ for every $h\in G$ (by left-invariance of the
uniform measure), hence by Schur's lemma $A_\lambda = c\,I$ for some scalar $c$.
Taking traces yields
\[
c\,d_\lambda \;=\; \mathrm{tr}(A_\lambda)
\;=\; \mathbb{E}_{g}\big[\mathrm{tr}(\rho_\lambda(g))\big],
\]
and orthogonality of traces (i.e., characters) of irreducible representations implies $\mathbb{E}_{g}[\rho_\lambda(g)]=0$
for every nontrivial $\lambda$, hence $c=0$ and therefore $A_\lambda=0$.
Consequently, the empirical average
$\frac{1}{|S|}\sum_{g\in S}\rho_\lambda(g)$ concentrates around $0$ as $|S|$ grows,
and the contribution of non-invariant components is attenuated by random averaging.

\paragraph{Isotypic components.}
Let $\mathcal{F}_\lambda \subseteq \mathcal{F}$ denote the $\lambda$-isotypic
subspace, i.e., the direct sum of all irreducible subrepresentations equivalent
to $V_\lambda$ in the decomposition
$\mathcal{F}\cong\bigoplus_{\lambda\in\widehat G}\mathbb{C}^{m_\lambda}\otimes V_\lambda$.
Equivalently, $\mathcal{F}_\lambda$ is the image of the orthogonal projector onto
the $\lambda$-block.
Any $f\in\mathcal{F}$ decomposes orthogonally as
\[
f \;=\; \sum_{\lambda\in\widehat G} f_\lambda,
\qquad
f_\lambda \in \mathcal{F}_\lambda,
\qquad
\langle f_\lambda, f_{\lambda'}\rangle_{L^2(\mathcal{X})} = 0 \;\;(\lambda\neq \lambda').
\]
In particular, the invariant subspace is the trivial isotypic component:
$\mathcal{F}^G = \mathcal{F}_{\mathrm{triv}}$.

\paragraph{Expected approximation error.}
Let $f\in\mathcal{F}$ with $\|f\|_{L^2(\mathcal{X})}\le 1$, and let $S$ be a multiset of
$|S|$ i.i.d.\ uniform samples from $G$. Since $\Pi_G f = f_{\mathrm{triv}}$ and
$\Pi_S$ acts blockwise, we have
\[
\Pi_S f - \Pi_G f \;=\; \sum_{\lambda\neq \mathrm{triv}} \Pi_S f_\lambda,
\]
and by orthogonality of distinct isotypic components (and unitarity of the action),
\[
\|\Pi_S f - \Pi_G f\|_{L^2(\mathcal{X})}^2
\;=\;
\sum_{\lambda\neq \mathrm{triv}} \|\Pi_S f_\lambda\|_{L^2(\mathcal{X})}^2.
\]
Taking expectation over the randomness of $S$ yields
\[
\mathbb{E}_S\big[\|\Pi_S f - \Pi_G f\|_{L^2(\mathcal{X})}^2\big]
\;=\;
\sum_{\lambda\neq \mathrm{triv}}
\mathbb{E}_S\big[\|\Pi_S f_\lambda\|_{L^2(\mathcal{X})}^2\big].
\]

Moreover, for any fixed component $u\in\mathcal{F}_\lambda$ with $\lambda\neq\mathrm{triv}$,
the random vectors $\{T_g u\}_{g}$ are mean-zero in $\mathcal{F}_\lambda$ and satisfy
$\|T_g u\|_{L^2(\mathcal{X})}=\|u\|_{L^2(\mathcal{X})}$. A direct variance computation gives
\[
\mathbb{E}_S\big[\|\Pi_S u\|_{L^2(\mathcal{X})}^2\big]
\;=\;
\frac{1}{|S|}\,\|u\|_{L^2(\mathcal{X})}^2,
\]
since cross terms vanish by $\mathbb{E}_{g}[T_g u]=0$ for $\lambda\neq\mathrm{triv}$.
Applying this with $u=f_\lambda$ and summing over $\lambda\neq\mathrm{triv}$ yields
\[
\mathbb{E}_S\big[\|\Pi_S f - \Pi_G f\|_{L^2(\mathcal{X})}^2\big]
\;=\;
\frac{1}{|S|}
\sum_{\lambda\neq \mathrm{triv}}\|f_\lambda\|_{L^2(\mathcal{X})}^2
\;\le\;
\frac{1}{|S|}\,\|f\|_{L^2(\mathcal{X})}^2
\;\le\;
\frac{1}{|S|}.
\]
In particular, the expected approximation error decays exactly as $1/|S|$,
uniformly over all $f\in\mathcal{F}$ with $\|f\|_{L^2(\mathcal{X})}\le 1$.

Therefore, random averaging over a small number of group elements preserves invariant
components exactly, while suppressing non-invariant components in expectation.
From a representation-theoretic viewpoint, this corresponds to averaging each
nontrivial irreducible block toward zero, with variance decreasing as the number
of sampled group elements increases.

\subsection{Uniform Bounds for Partial Data Augmentation}

In this subsection, we present a result showing that a single instance of
\emph{partial data augmentation}, constructed using a random subset
$S\subseteq G$, can be reused to obtain guarantees that hold uniformly over the
entire function space $\mathcal F$, with high probability. In contrast to
expectation-based bounds that apply to a fixed estimator, the results below
control the approximation error of partial augmentation simultaneously for all
functions in $\mathcal F$, making them suitable for algorithm-agnostic and
reusable augmentation schemes.

\begin{theorem}[Partial data augmentation uniformly over $\mathcal F$ (high probability)]
\label{thm:partial-augmentation-uniform-highprob}
Let $(\mathcal X,\mu)$ be a measured space and let $\mathcal F\subseteq L^2(\mathcal X,\mu)$
be a finite-dimensional subspace with $\dim(\mathcal F)=r$.
Assume a group $G$ acts isometrically on $(\mathcal X,\mu)$ and that $\mathcal F$ is
closed under the induced action. Let $T_g:\mathcal F\to\mathcal F$ denote the lifted
(unitary) operators and let
\[
\Pi_G \;:=\; \mathbb E_{g\sim G}[T_g]
\qquad\text{and}\qquad
\Pi_S \;:=\; \frac{1}{|S|}\sum_{g\in S}T_g,
\]
where $S=\{g_1,\dots,g_{|S|}\}$ is a multiset of $|S|$ i.i.d.\ uniform samples from $G$.
Then $\Pi_G$ is the $L^2(\mathcal{X})$-orthogonal projector from $\mathcal F$ onto the invariant
subspace $\mathcal F^G:=\{f\in\mathcal F:T_g f=f\ \forall g\in G\}$.

Fix any $\delta\in(0,1)$ and any radius $B>0$. With probability at least $1-\delta$
over the draw of $S$, the following holds \emph{simultaneously for all}
$f\in\mathcal F$ with $\|f\|_{L^2(\mathcal{X})}\le B$:
\[
\|\Pi_S f-\Pi_G f\|_{L^2(\mathcal{X})}
\;\le\;
\|\Pi_S-\Pi_G\|_{\mathrm{op}}\;\|f\|_{L^2(\mathcal{X})}
\;\le\;
C\,B\,\sqrt{\frac{\log\!\big(\min\{r,|G|\}/\delta\big)}{|S|}},
\]
where $\|\cdot\|_{\mathrm{op}}$ denotes the operator norm on $\mathcal F$ induced by
$\|\cdot\|_{L^2(\mathcal{X})}$, and $C>0$ is a universal constant. Equivalently,
\[
\sup_{\substack{f\in\mathcal F\\ \|f\|_{L^2(\mathcal{X})}\le B}}
\|\Pi_S f-\Pi_G f\|_{L^2(\mathcal{X})}
\;\le\;
C\,B\,\sqrt{\frac{\log\!\big(\min\{r,|G|\}/\delta\big)}{|S|}}.
\]

In particular, if one performs \emph{partial data augmentation} using the fixed set $S$
and then applies the corresponding augmentation operator $\Pi_S$ to any predictor
$f\in\mathcal F$, the augmentation error relative to full augmentation (i.e., $\Pi_G$)
is uniformly controlled as above.
\end{theorem}

\begin{proof}
We work on the Hilbert space $(\mathcal F,\langle\cdot,\cdot\rangle)$ with the
$L^2(\mathcal{X})$ inner product. Since the action of $G$ on $(\mathcal X,\mu)$ is
isometric and $\mathcal F$ is closed under the induced action, each lifted map
$T_g:\mathcal F\to\mathcal F$ is unitary:
\[
\|T_g f\|_{L^2(\mathcal{X})}=\|f\|_{L^2(\mathcal{X})}\qquad\text{for all }f\in\mathcal F,\ g\in G.
\]
In particular, $\|T_g\|_{\mathrm{op}}=1$ and $\|\Pi_S\|_{\mathrm{op}}\le 1$,
$\|\Pi_G\|_{\mathrm{op}}\le 1$.

\paragraph{Step 1: $\Pi_G$ is the orthogonal projector onto $\mathcal F^G$.}
By definition, $\Pi_G=\mathbb E_g[T_g]$ is a bounded linear operator on $\mathcal F$.
For any $h\in\mathcal F^G$ we have $T_g h=h$ for all $g$, hence $\Pi_G h=h$.
Conversely, for any $f\in\mathcal F$ and any $g_0\in G$,
left-invariance of the Haar/uniform measure implies
\[
T_{g_0}\Pi_G f \;=\; T_{g_0}\,\mathbb E_g[T_g f]
\;=\; \mathbb E_g[T_{g_0 g} f]
\;=\; \mathbb E_g[T_g f]
\;=\; \Pi_G f,
\]
so $\Pi_G f\in\mathcal F^G$. Thus $\mathrm{range}(\Pi_G)=\mathcal F^G$ and
$\Pi_G$ acts as the identity on $\mathcal F^G$. Moreover, since each $T_g$ is unitary,
$\Pi_G$ is self-adjoint:
\[
\langle \Pi_G f, h\rangle
=
\mathbb E_g\langle T_g f,h\rangle
=
\mathbb E_g\langle f,T_g^{-1} h\rangle
=
\mathbb E_g\langle f,T_g h\rangle
=
\langle f,\Pi_G h\rangle,
\]
where we used $T_g^{-1}=T_{g^{-1}}$ and invariance of the uniform/Haar measure under
inversion. A self-adjoint idempotent operator is an orthogonal projector, hence
$\Pi_G$ is the $L^2(\mathcal{X})$-orthogonal projector onto $\mathcal F^G$.

\paragraph{Step 2: Uniform control reduces to the operator norm.}
For any $f\in\mathcal F$,
\[
\|\Pi_S f-\Pi_G f\|_{L^2(\mathcal{X})}
\;\le\;
\|\Pi_S-\Pi_G\|_{\mathrm{op}}\;\|f\|_{L^2(\mathcal{X})}.
\]
Therefore, it suffices to bound $\|\Pi_S-\Pi_G\|_{\mathrm{op}}$ with high probability.

\paragraph{Step 3: Block decomposition and reduction to nontrivial irreducibles.}
Fix an orthonormal basis of $\mathcal F$ that block-diagonalizes the unitary
representation $g\mapsto T_g$ into irreducible components (finite-group harmonic
decomposition):
\[
\mathcal F \;\cong\; \bigoplus_{\lambda\in\widehat G}\mathbb C^{m_\lambda}\otimes V_\lambda,
\]
where $V_\lambda$ is an irreducible representation of dimension $d_\lambda$.
Under this change of basis, each $T_g$ becomes block-diagonal with blocks
$I_{m_\lambda}\otimes \rho_\lambda(g)$.
On the trivial block $\lambda=\mathrm{triv}$, $\rho_{\mathrm{triv}}(g)=1$ for all $g$,
so $\Pi_S$ and $\Pi_G$ coincide (both equal the identity) on $\mathcal F^G$.
Thus
\[
\Pi_S-\Pi_G
\;=\;
\bigoplus_{\lambda\neq \mathrm{triv}}
\Big(I_{m_\lambda}\otimes \Big(\frac{1}{|S|}\sum_{g\in S}\rho_\lambda(g)\Big)\Big),
\]
and hence
\[
\|\Pi_S-\Pi_G\|_{\mathrm{op}}
=
\max_{\lambda\neq \mathrm{triv}}
\left\|\frac{1}{|S|}\sum_{g\in S}\rho_\lambda(g)\right\|_{\mathrm{op}}.
\]

\paragraph{Step 4: Matrix concentration for each nontrivial irreducible block.}
Fix $\lambda\neq\mathrm{triv}$ and define i.i.d.\ random matrices
$X_i:=\rho_\lambda(g_i)\in\mathbb C^{d_\lambda\times d_\lambda}$ for $g_i\sim G$.
Since $\rho_\lambda$ is unitary, $\|X_i\|_{\mathrm{op}}=1$. Moreover,
\[
\mathbb E[X_i]
=
\mathbb E_{g\sim G}[\rho_\lambda(g)]
=
0,
\]
because $\mathbb E_g[\rho_\lambda(g)]$ is an intertwiner from $\rho_\lambda$ to itself,
hence by Schur's lemma it must be a scalar multiple of the identity; taking traces gives
that scalar equals $\frac{1}{d_\lambda}\mathbb E_g[\chi_\lambda(g)]$, which is $0$ for
any nontrivial irreducible representation.
Therefore, $\{X_i\}_{i=1}^{|S|}$ are independent, mean-zero, and satisfy
$\|X_i\|_{\mathrm{op}}\le 1$.

By a standard matrix Bernstein inequality (for sums of independent mean-zero matrices),
for any $t>0$,
\[
\mathbb P\!\left(
\left\|\frac{1}{|S|}\sum_{i=1}^{|S|}X_i\right\|_{\mathrm{op}}
\ge t
\right)
\;\le\;
2d_\lambda\exp\!\Big(-c\,|S|\,t^2\Big),
\]
for a universal constant $c>0$ (using the crude variance bound
$\|\mathbb E[X_iX_i^\ast]\|_{\mathrm{op}}\le 1$ and similarly for $X_i^\ast X_i$). For more details, see \citep{tropp2012user}.

\paragraph{Step 5: Union bound over irreducible blocks.}
Taking a union bound over all $\lambda\neq\mathrm{triv}$ and using that
$\sum_{\lambda\in\widehat G} d_\lambda^2 = |G|$ (finite groups) and $d_\lambda\le \sqrt{|G|}$,
we have
\begin{align*}
\mathbb P\!\left(\|\Pi_S-\Pi_G\|_{\mathrm{op}}\ge t\right) &\le
\sum_{\lambda\neq\mathrm{triv}}
2d_\lambda \exp(-c|S|t^2)
\;\\ &\le\;
2\Big(\sum_{\lambda\in\widehat G} d_\lambda\Big)\exp(-c|S|t^2)
\; \\ & \le\;
2\,\min\{r,|G|\}\,\exp(-c|S|t^2),
\end{align*}
where we used $\sum_{\lambda} d_\lambda\le \sum_\lambda d_\lambda^2 = |G|$ and also
$\sum_\lambda d_\lambda\le r$ since the total dimension of the representation on
$\mathcal F$ is $r$.

Choosing
\[
t
=
C\sqrt{\frac{\log(\min\{r,|G|\}/\delta)}{|S|}}
\]
for a sufficiently large universal constant $C>0$ makes the right-hand side at most $\delta$.
Thus, with probability at least $1-\delta$,
\[
\|\Pi_S-\Pi_G\|_{\mathrm{op}}
\;\le\;
C\sqrt{\frac{\log(\min\{r,|G|\}/\delta)}{|S|}}.
\]

\paragraph{Step 6: Conclude the uniform bound over $\|f\|\le B$.}
On this event, for every $f\in\mathcal F$ with $\|f\|_{L^2(\mathcal{X})}\le B$,
\[
\|\Pi_S f-\Pi_G f\|_{L^2(\mathcal{X})}
\;\le\;
\|\Pi_S-\Pi_G\|_{\mathrm{op}}\,\|f\|_{L^2(\mathcal{X})}
\;\le\;
C\,B\,\sqrt{\frac{\log(2\min\{r,|G|\}/\delta)}{|S|}},
\]
which is exactly the desired statement.
\end{proof}

\subsection{Baseline Excess Risk of Projection Estimators (No Augmentation)}
\label{app:baseline-projection-excess}

Let $\Pi_{\mathcal F}$ denote the $L^2(\mathcal{X})$-orthogonal projection onto
$\mathcal F\subset L^2(\mathcal X,\mu)$.

\begin{lemma}[Baseline excess risk: density estimation]
\label{lem:baseline-density-excess}
Let $\mathcal{F}\subset L^2(\mathcal X,\mu)$ be $r$-dimensional with orthonormal
basis $\{\phi_\ell\}_{\ell=1}^r$. Let $x_1,\dots,x_n$ be i.i.d.\ samples drawn from
an unknown density $f^\star$ with respect to $\mu$, and assume
$f^\star\in L^2(\mathcal{X})\cap L^\infty(\mu)$. Define
\[
\widehat{\theta}_\ell := \frac{1}{n}\sum_{i=1}^n \phi_\ell(x_i),
\qquad
\widehat{f} := \sum_{\ell=1}^r \widehat{\theta}_\ell\,\phi_\ell .
\]
Then the expected \emph{excess} $L^2(\mathcal{X})$ error over $\mathcal F$ satisfies
\[
\mathbb{E}\big[\|\widehat{f}-\Pi_{\mathcal F} f^\star\|_{L^2(\mathcal{X})}^2\big]
\;\le\;
\frac{\|f^\star\|_\infty}{n}\,r.
\]
\end{lemma}

\begin{proof}
Write $\Pi_{\mathcal F} f^\star = \sum_{\ell=1}^r \theta_\ell \phi_\ell$ with
$\theta_\ell=\langle f^\star,\phi_\ell\rangle=\mathbb E_{x\sim f^\star}[\phi_\ell(x)]$.
By orthonormality,
\[
\|\widehat f-\Pi_{\mathcal F}f^\star\|_{L^2(\mathcal{X})}^2
=
\sum_{\ell=1}^r (\widehat\theta_\ell-\theta_\ell)^2.
\]
Taking expectation gives
\[
\mathbb E\big[\|\widehat f-\Pi_{\mathcal F}f^\star\|_{L^2(\mathcal{X})}^2\big]
=
\sum_{\ell=1}^r \mathrm{Var}(\widehat\theta_\ell).
\]
Since $\widehat\theta_\ell$ is an empirical mean,
$\mathrm{Var}(\widehat\theta_\ell)=\frac{1}{n}\mathrm{Var}_{x\sim f^\star}(\phi_\ell(x))
\le \frac{1}{n}\mathbb E_{x\sim f^\star}[\phi_\ell(x)^2]$.
Moreover,
\[
\mathbb E_{x\sim f^\star}[\phi_\ell(x)^2]
=\int_{\mathcal X}\phi_\ell(x)^2 f^\star(x)\,d\mu(x)
\le \|f^\star\|_\infty \int_{\mathcal X}\phi_\ell(x)^2\,d\mu(x)
=\|f^\star\|_\infty,
\]
and summing over $\ell\in[r]$ yields the claim.
\end{proof}

\begin{lemma}[Baseline excess risk: regression]
\label{lem:baseline-regression-excess}
Let $\mathcal{F}\subset L^2(\mathcal X,\mu)$ be $r$-dimensional with orthonormal
basis $\{\phi_\ell\}_{\ell=1}^r$. Let $(x_i,y_i)_{i=1}^n$ be i.i.d.\ samples where
$x_i\sim\mu$ and
\[
y_i = f^\star(x_i) + \varepsilon_i,
\qquad
\mathbb{E}[\varepsilon_i]=0,
\qquad
\mathbb{E}[\varepsilon_i^2]=\sigma^2,
\]
with $\varepsilon_i$ independent of $x_i$. Assume $f^\star\in L^2(\mathcal{X})\cap L^\infty(\mu)$.
Define
\[
\widehat{\beta}_\ell := \frac{1}{n}\sum_{i=1}^n y_i\,\phi_\ell(x_i),
\qquad
\widehat{f} := \sum_{\ell=1}^r \widehat{\beta}_\ell\,\phi_\ell .
\]
Then the expected \emph{excess} $L^2(\mathcal{X})$ error over $\mathcal F$ satisfies
\[
\mathbb{E}\big[\|\widehat{f}-\Pi_{\mathcal F} f^\star\|_{L^2(\mathcal{X})}^2\big]
\;\le\;
\frac{\|f^\star\|_\infty^2+\sigma^2}{n}\,r.
\]
Equivalently, for the squared-loss population risk
$R(f):=\mathbb E[(y-f(x))^2]$,
\[
\mathbb E\big[R(\widehat f)-\inf_{f\in\mathcal F}R(f)\big]
\;=\;
\mathbb{E}\big[\|\widehat{f}-\Pi_{\mathcal F} f^\star\|_{L^2(\mathcal{X})}^2\big]
\;\le\;
\frac{\|f^\star\|_\infty^2+\sigma^2}{n}\,r.
\]
\end{lemma}

\begin{proof}
Let $\Pi_{\mathcal F} f^\star = \sum_{\ell=1}^r \beta_\ell\phi_\ell$, where
\[
\beta_\ell=\langle f^\star,\phi_\ell\rangle_{L^2(\mathcal{X})}
=\mathbb E_{x\sim\mu}[f^\star(x)\phi_\ell(x)]
=\mathbb E[y\,\phi_\ell(x)].
\]
As before,
\[
\|\widehat f-\Pi_{\mathcal F}f^\star\|_{L^2(\mathcal{X})}^2
=
\sum_{\ell=1}^r(\widehat\beta_\ell-\beta_\ell)^2,
\qquad
\mathbb E\big[\|\widehat f-\Pi_{\mathcal F}f^\star\|_{L^2(\mathcal{X})}^2\big]
=
\sum_{\ell=1}^r \mathrm{Var}(\widehat\beta_\ell).
\]
Since $\widehat\beta_\ell$ is an empirical mean,
$\mathrm{Var}(\widehat\beta_\ell)=\frac{1}{n}\mathrm{Var}(y\phi_\ell(x))
\le \frac{1}{n}\mathbb E[y^2\phi_\ell(x)^2]$.
Moreover, using $\mathbb E[y^2\mid x]=f^\star(x)^2+\sigma^2$,
\[
\mathbb E[y^2\phi_\ell(x)^2]
=
\mathbb E_{x\sim\mu}[(f^\star(x)^2+\sigma^2)\phi_\ell(x)^2]
\le (\|f^\star\|_\infty^2+\sigma^2)\,\mathbb E_{x\sim\mu}[\phi_\ell(x)^2]
=
\|f^\star\|_\infty^2+\sigma^2,
\]
and summing over $\ell\in[r]$ gives the stated bound.

Finally, for squared loss $R(f)=\mathbb E[(y-f(x))^2]$ with
$y=f^\star(x)+\varepsilon$ and $\mathbb E[\varepsilon\mid x]=0$, one has the standard
identity
\[
R(f)-R(\Pi_{\mathcal F}f^\star)=\|f-\Pi_{\mathcal F}f^\star\|_{L^2(\mathcal{X})}^2,
\]
which yields the excess-risk equality.
\end{proof}

\section{Proof of Theorem~\ref{thm:partial-augmentation-excess}}

\begin{proof}
We work with the $L^2(\mathcal{X})$ inner product. Since $\mathcal{F}$ is closed under the
group action and the action is isometric, the averaged operator
$\Pi_G := \mathbb E_{g}[T_g]$ is the orthogonal projector onto $\mathcal{F}^G$.
Let $\Pi_{\mathcal F}$ denote the $L^2(\mathcal{X})$-orthogonal projection onto
$\mathcal F\subset L^2(\mathcal X,\mu)$, and similarly define $\Pi_{\mathcal F^G}$.
Hence for any $h\in\mathcal F$, $\Pi_G h = \Pi_{\mathcal F^G} h$, and in particular
$\Pi_G(\Pi_{\mathcal F} f^\star)=\Pi_{\mathcal F^G}f^\star$.

\paragraph{Step 1: Define an intermediate target in $\mathcal F$.}
Let
\[
f_{\mathcal F} \;:=\; \Pi_{\mathcal F} f^\star \in \mathcal F,
\qquad
f_{\mathrm{inv}} \;:=\; \Pi_{\mathcal F^G} f^\star = \Pi_G f_{\mathcal F}\in \mathcal F^G.
\]
We will bound the excess risk $\|\widehat f_S-f_{\mathrm{inv}}\|_{L^2(\mathcal{X})}^2$.

\paragraph{Step 2: Estimation inside the invariant subspace.}
Consider the ideal estimator that uses \emph{full} group averaging (equivalently,
projects onto $\mathcal F^G$ with respect to $\mu$). Denote it by
$\widehat f_{\mathrm{inv}}$. By Lemma~\ref{lem:baseline-density-excess} applied to
the $r_{\mathrm{inv}}$-dimensional space $\mathcal F^G$ (with its orthonormal basis),
\[
\mathbb E\big[\|\widehat f_{\mathrm{inv}}-f_{\mathrm{inv}}\|_{L^2(\mathcal{X})}^2\big]
\;\le\;
\frac{\|f^\star\|_\infty}{n}\,r_{\mathrm{inv}}
\]
in the density estimation setting. In the regression setting, the analogous bound
follows from Lemma~\ref{lem:baseline-regression-excess}:
\[
\mathbb E\big[\|\widehat f_{\mathrm{inv}}-f_{\mathrm{inv}}\|_{L^2(\mathcal{X})}^2\big]
\;\le\;
\frac{\|f^\star\|_\infty^2+\sigma^2}{n}\,r_{\mathrm{inv}}.
\]

\paragraph{Step 3: Replacing full averaging by averaging over $S$.}
Let $\Pi_S := \frac{1}{|S|}\sum_{g\in S} T_g$ denote the empirical averaging operator.
By the construction of the augmented projection estimator and linearity, we may write
\[
\widehat f_S \;=\; \Pi_S \widehat f,
\qquad
\widehat f_{\mathrm{inv}} \;=\; \Pi_G \widehat f,
\]
for the (non-augmented) projection estimator $\widehat f\in\mathcal F$ built from
the base samples (density or regression). Conditioning on the base data, we may
apply the random-averaging identity from
Appendix~\ref{app:approx-projection-random-averaging} to obtain
\[
\mathbb E_S\big[\|\widehat f_S-\widehat f_{\mathrm{inv}}\|_{L^2(\mathcal{X})}^2 \,\big|\, x_{1:n}\big]
=
\mathbb E_S\big[\|(\Pi_S-\Pi_G)\widehat f\|_{L^2(\mathcal{X})}^2 \,\big|\, x_{1:n}\big]
=
\frac{1}{|S|}\,\|\widehat f-\Pi_G\widehat f\|_{L^2(\mathcal{X})}^2.
\]
Taking expectation over the base data and using
\(\Pi_G=\Pi_{\mathcal F^G}\) on \(\mathcal F\), we obtain
\[
\mathbb E\big[
\|\widehat f_S-\widehat f_{\mathrm{inv}}\|_{L^2(\mathcal X)}^2
\big]
=
\frac{1}{|S|}
\mathbb E\big[
\|\widehat f-\Pi_{\mathcal F^G}\widehat f\|_{L^2(\mathcal X)}^2
\big].
\]
Since \(f_{\mathcal F}\in \mathcal F^G\), we have
\[
(I-\Pi_{\mathcal F^G})f_{\mathcal F}=0.
\]
Therefore,
\[
\widehat f-\Pi_{\mathcal F^G}\widehat f
=
(I-\Pi_{\mathcal F^G})\widehat f
=
(I-\Pi_{\mathcal F^G})(\widehat f-f_{\mathcal F}).
\]
Since \(I-\Pi_{\mathcal F^G}\) is an orthogonal projector, it is non-expansive in
\(L^2(\mathcal X)\). Hence
\[
\|\widehat f-\Pi_{\mathcal F^G}\widehat f\|_{L^2(\mathcal X)}^2
\le
\|\widehat f-f_{\mathcal F}\|_{L^2(\mathcal X)}^2.
\]
Consequently,
\[
\mathbb E\big[
\|\widehat f_S-\widehat f_{\mathrm{inv}}\|_{L^2(\mathcal X)}^2
\big]
\le
\frac{1}{|S|}
\mathbb E\big[
\|\widehat f-f_{\mathcal F}\|_{L^2(\mathcal X)}^2
\big].
\]
By Lemma~\ref{lem:baseline-density-excess} in the density-estimation setting, and
by Lemma~\ref{lem:baseline-regression-excess} in the regression setting,
\[
\mathbb E\big[
\|\widehat f-f_{\mathcal F}\|_{L^2(\mathcal X)}^2
\big]
\lesssim
\frac{r}{n}.
\]
Therefore,
\[
\mathbb E\big[
\|\widehat f_S-\widehat f_{\mathrm{inv}}\|_{L^2(\mathcal X)}^2
\big]
\lesssim
\frac{r}{n|S|}.
\]

\paragraph{Step 4: Combine errors.}
Finally, by the inequality \(\|a+b\|^2\le 2\|a\|^2+2\|b\|^2\),
\[
\mathbb E\big[
\|\widehat f_S-f_{\mathrm{inv}}\|_{L^2(\mathcal X)}^2
\big]
\le
2\,\mathbb E\big[
\|\widehat f_S-\widehat f_{\mathrm{inv}}\|_{L^2(\mathcal X)}^2
\big]
+
2\,\mathbb E\big[
\|\widehat f_{\mathrm{inv}}-f_{\mathrm{inv}}\|_{L^2(\mathcal X)}^2
\big].
\]
Substituting the bounds from Steps~2--3 yields, for density estimation,
\[
\mathbb E\big[
\|\widehat f_S-\Pi_{\mathcal F^G}f^\star\|_{L^2(\mathcal X)}^2
\big]
\lesssim
\frac{\|f^\star\|_\infty}{n}\,r_{\mathrm{inv}}
+
\frac{r}{n|S|}.
\]
Similarly, in the regression setting with additive zero-mean noise of variance
\(\sigma^2\), we obtain
\[
\mathbb E\big[
\|\widehat f_S-\Pi_{\mathcal F^G}f^\star\|_{L^2(\mathcal X)}^2
\big]
\lesssim
\frac{\|f^\star\|_\infty^2+\sigma^2}{n}\,r_{\mathrm{inv}}
+
\frac{r}{n|S|}.
\]
Here we used that \(f_{\mathrm{inv}}=\Pi_{\mathcal F^G}f^\star\), since
\(f_{\mathcal F}\in \mathcal F^G\) under the invariance assumption. This completes
the proof.
\end{proof}

\begin{remark}[Optimality up to constants]
The proven upper bound is optimal up to absolute constants for random \(S\).
Indeed, the proof is based on an orthogonal decomposition into the invariant
component and the residual non-invariant component. By the Pythagorean theorem,
these two components contribute additively to the \(L^2\) error. The invariant
component yields the usual \(r_{\mathrm{inv}}/n\) term, while random subset
averaging reduces the non-invariant contribution by a factor of \(|S|\), yielding
the \(r/(n|S|)\) term. Thus, in general, neither term can be improved beyond
absolute constants.
\end{remark}

\section{Proof of Theorem~\ref{thm:partial-augmentation-excess-logG-compact}}

\begin{proof}
We prove the density-estimation case; the regression case is identical, with the
baseline projection-estimator bound over $\mathcal F^G$ replaced by the
corresponding regression bound. 

\paragraph{Step 0: Notation.}
Let $\widehat f\in\mathcal F$ denote the (unaugmented) projection estimator based
on the samples $x_1,\dots,x_n$, and let $\widehat f_S\in\mathcal F$ denote the
projection estimator obtained from the \emph{partially augmented} samples
$\{g^{-1} x_i : i\in[n],\, g\in S\}$.
As shown in the preliminaries (closure of $\mathcal F$ under the action and the
orthonormal-basis identification), partial data augmentation by $S$ corresponds
to applying the averaging operator $\Pi_S:=\frac{1}{|S|}\sum_{g\in S}T_g$ to the
(unaugmented) coefficient vector. Equivalently, at the function level,
\begin{equation}
\label{eq:fS-is-PiS-fhat}
\widehat f_S \;=\; \Pi_S \widehat f ,
\end{equation}
where $T_g$ denotes the lifted unitary action on $\mathcal F$.
Moreover, the full-group averaging operator
$\Pi_G:=\mathbb E_{g\sim G}[T_g]$ is the $L^2(\mathcal{X})$-orthogonal projector onto
$\mathcal F^G$.

\paragraph{Step 1: Decompose the excess error.}
Let \(f_{\mathrm{inv}}:=\Pi_{\mathcal F^G} f^\star\). Under the invariance
assumption, \(f_{\mathcal F}:=\Pi_{\mathcal F}f^\star\) belongs to
\(\mathcal F^G\), and hence \(f_{\mathrm{inv}}=f_{\mathcal F}\).
Using \(\widehat f_S=\Pi_S\widehat f\) from~Equation~\eqref{eq:fS-is-PiS-fhat}
and adding and subtracting \(\Pi_G\widehat f\), we have
\[
\widehat f_S-f_{\mathrm{inv}}
=
\big(\Pi_S\widehat f-\Pi_G\widehat f\big)
+
\big(\Pi_G\widehat f-f_{\mathrm{inv}}\big).
\]
Therefore, by \(\|a+b\|^2\le 2\|a\|^2+2\|b\|^2\),
\begin{equation}
\label{eq:split}
\|\widehat f_S-f_{\mathrm{inv}}\|_{L^2(\mathcal X)}^2
\le
2\|(\Pi_S-\Pi_G)\widehat f\|_{L^2(\mathcal X)}^2
+
2\|\Pi_G\widehat f-f_{\mathrm{inv}}\|_{L^2(\mathcal X)}^2.
\end{equation}

\paragraph{Step 2: Control the partial-augmentation error via the operator norm.}
We first use the invariance of \(f_{\mathcal F}\) to center the partial-augmentation
error around the baseline estimation error. Since \(f_{\mathcal F}\in\mathcal F^G\),
we have
\[
\Pi_S f_{\mathcal F}=f_{\mathcal F},
\qquad
\Pi_G f_{\mathcal F}=f_{\mathcal F}.
\]
Hence
\[
(\Pi_S-\Pi_G)f_{\mathcal F}=0,
\]
and therefore
\[
(\Pi_S-\Pi_G)\widehat f
=
(\Pi_S-\Pi_G)(\widehat f-f_{\mathcal F}).
\]
Consequently,
\begin{equation}
\label{eq:opnorm-centered}
\|(\Pi_S-\Pi_G)\widehat f\|_{L^2(\mathcal X)}^2
\le
\|\Pi_S-\Pi_G\|_{\mathrm{op}}^2
\|\widehat f-f_{\mathcal F}\|_{L^2(\mathcal X)}^2.
\end{equation}

By Theorem~\ref{thm:partial-augmentation-uniform-highprob}, for any
\(\delta\in(0,1)\), with probability at least \(1-\delta\) over the draw of \(S\),
\begin{equation}
\label{eq:PiS-PiG-op}
\|\Pi_S-\Pi_G\|_{\mathrm{op}}^2
\le
C_0\,\frac{\log\!\big(\min\{r,|G|\}/\delta\big)}{|S|}
\end{equation}
for a universal constant \(C_0>0\). Combining
Equations~\eqref{eq:opnorm-centered} and~\eqref{eq:PiS-PiG-op}, on the same event,
\begin{equation}
\label{eq:aug-term}
\|(\Pi_S-\Pi_G)\widehat f\|_{L^2(\mathcal X)}^2
\le
C_0\,\frac{\log\!\big(\min\{r,|G|\}/\delta\big)}{|S|}
\|\widehat f-f_{\mathcal F}\|_{L^2(\mathcal X)}^2.
\end{equation}

\paragraph{Step 3: Baseline estimation inside the invariant subspace.}
Since \(\Pi_G\) is the orthogonal projector onto \(\mathcal F^G\), and
\(f_{\mathrm{inv}}=f_{\mathcal F}\in\mathcal F^G\), we have
\[
\Pi_G\widehat f-f_{\mathrm{inv}}
=
\Pi_G(\widehat f-f_{\mathcal F}).
\]
Thus, by non-expansiveness of the orthogonal projector \(\Pi_G\),
\[
\|\Pi_G\widehat f-f_{\mathrm{inv}}\|_{L^2(\mathcal X)}^2
\le
\|\widehat f-f_{\mathcal F}\|_{L^2(\mathcal X)}^2.
\]
Moreover, the refined invariant projection-estimator bound from
Step~2 of the proof gives, in the density-estimation setting,
\begin{equation}
\label{eq:inv-rate}
\mathbb E\big[
\|\Pi_G\widehat f-f_{\mathrm{inv}}\|_{L^2(\mathcal X)}^2
\big]
\le
C_1\,\frac{\|f^\star\|_\infty}{n}\,r_{\mathrm{inv}}.
\end{equation}
In the regression setting with additive zero-mean noise of variance \(\sigma^2\),
the same argument gives
\[
\mathbb E\big[
\|\Pi_G\widehat f-f_{\mathrm{inv}}\|_{L^2(\mathcal X)}^2
\big]
\le
C_1\,\frac{\|f^\star\|_\infty^2+\sigma^2}{n}\,r_{\mathrm{inv}}.
\]

\paragraph{Step 4: Take conditional expectation over the data.}
Taking conditional expectation of~\eqref{eq:split} given \(S\), and using
Equation~\eqref{eq:aug-term}, we obtain, on the event in
Equation~\eqref{eq:PiS-PiG-op},
\begin{align*}
\mathbb E\Big[
\|\widehat f_S-f_{\mathrm{inv}}\|_{L^2(\mathcal X)}^2
\,\Big|\, S
\Big]
&\le
2C_0\,\frac{\log\!\big(\min\{r,|G|\}/\delta\big)}{|S|}
\mathbb E\big[
\|\widehat f-f_{\mathcal F}\|_{L^2(\mathcal X)}^2
\big]
\\
&\qquad
+
2\,\mathbb E\big[
\|\Pi_G\widehat f-f_{\mathrm{inv}}\|_{L^2(\mathcal X)}^2
\big].
\end{align*}
By Lemma~\ref{lem:baseline-density-excess}, the baseline projection estimator
satisfies
\[
\mathbb E\big[
\|\widehat f-f_{\mathcal F}\|_{L^2(\mathcal X)}^2
\big]
\le
C_2\,\frac{\|f^\star\|_\infty}{n}\,r
\]
in the density-estimation setting. Combining this with
Equation~\eqref{eq:inv-rate} gives
\[
\mathbb E\!\left[
\big\|\widehat f_S-\Pi_{\mathcal F^G}f^\star\big\|_{L^2(\mathcal X)}^2
\;\middle|\; S
\right]
\le
C\|f^\star\|_\infty
\left(
\frac{r_{\mathrm{inv}}}{n}
+
\frac{r\,\log\!\big(\min\{r,|G|\}/\delta\big)}{n|S|}
\right).
\]
Similarly, in the regression setting with additive zero-mean noise of variance
\(\sigma^2\), Lemma~\ref{lem:baseline-regression-excess} yields
\[
\mathbb E\big[
\|\widehat f-f_{\mathcal F}\|_{L^2(\mathcal X)}^2
\big]
\le
C_2\,\frac{\|f^\star\|_\infty^2+\sigma^2}{n}\,r,
\]
and therefore
\[
\mathbb E\!\left[
\big\|\widehat f_S-\Pi_{\mathcal F^G}f^\star\big\|_{L^2(\mathcal X)}^2
\;\middle|\; S
\right]
\le
C\big(\|f^\star\|_\infty^2+\sigma^2\big)
\left(
\frac{r_{\mathrm{inv}}}{n}
+
\frac{r\,\log\!\big(\min\{r,|G|\}/\delta\big)}{n|S|}
\right).
\]
This completes the proof.
\end{proof}

\section{Proof of Theorem~\ref{thm:informal-exact-invariance}}

We next state the formal version of Theorem~\ref{thm:informal-exact-invariance} and provide its proof.

\begin{theorem}[Exact invariance forces full augmentation]
\label{thm:exact-invariance-forces-full-group}
Let $G$ be a finite group and let $\mathcal F$ be a finite-dimensional real (or complex)
Hilbert space carrying a unitary representation $\rho:G\to U(\mathcal F)$.
For a multiset $S\subseteq G$, define the averaging operator
\[
\Pi_S \;:=\; \frac{1}{|S|}\sum_{g\in S}\rho(g),
\qquad
\Pi_G \;:=\; \frac{1}{|G|}\sum_{g\in G}\rho(g).
\]
Assume $\mathcal F$ is \emph{representation-complete} in the sense that every irreducible
representation of $G$ appears as a subrepresentation of $\mathcal F$ (equivalently,
for every $\lambda\in\widehat G$ the multiplicity $m_\lambda(\mathcal F)\ge 1$).

Suppose that partial data augmentation using $S$ yields an estimator that is
\emph{exactly $G$-invariant for all inputs}, i.e.,
\begin{equation}
\label{eq:exact-invariance-PiS}
\rho(h)\Pi_S \;=\; \Pi_S
\qquad\text{for all }h\in G,
\end{equation}
or equivalently $\Pi_S=\Pi_G$ on $\mathcal F$.

Then $S$ must be uniform over the whole group in the following sense:
if we write $S$ as a multiset with multiplicity function $m_S:G\to\mathbb Z_{\ge 0}$, then
\[
m_S(g)\ \text{is constant over }g\in G,
\]
and in particular $S$ contains each group element equally often. Consequently, if
$S$ is a \emph{subset} (no repetitions), then necessarily $S=G$.

Equivalently, under the representation-completeness assumption, the condition
$\Pi_S=\Pi_G$ holds if and only if the Fourier coefficients of the uniform measure on $S$
vanish on all nontrivial irreducible representations:
\[
\frac{1}{|S|}\sum_{g\in S}\rho_\lambda(g)=0
\qquad\text{for all }\lambda\in\widehat G\setminus\{\mathrm{triv}\},
\]
which forces $S$ to be (multi)setwise uniform over $G$.
\end{theorem}

\begin{proof}
Let $G$ be a finite group and let $\rho:G\to U(\mathcal F)$ be a unitary
representation on the finite-dimensional Hilbert space $\mathcal F$.
By assumption, $\mathcal F$ is representation-complete, meaning that every
irreducible representation of $G$ appears with positive multiplicity in $\mathcal F$.

Recall the averaging operators
\[
\Pi_S := \frac{1}{|S|}\sum_{g\in S}\rho(g),
\qquad
\Pi_G := \frac{1}{|G|}\sum_{g\in G}\rho(g),
\]
where $S$ is viewed as a multiset.

\paragraph{Irreducible decomposition.}
By the Peter--Weyl theorem for finite groups, $\mathcal F$ admits an orthogonal
decomposition into irreducible components
\[
\mathcal F \;\cong\; \bigoplus_{\lambda\in\widehat G}
\mathbb C^{m_\lambda}\otimes V_\lambda,
\]
where $V_\lambda$ is an irreducible representation of dimension $d_\lambda$ and
$m_\lambda\ge 1$ by representation-completeness.
With respect to this decomposition, the representation $\rho$ is block-diagonal:
\[
\rho(g)
\;=\;
\bigoplus_{\lambda\in\widehat G}
I_{m_\lambda}\otimes\rho_\lambda(g),
\]
and hence both $\Pi_S$ and $\Pi_G$ are block-diagonal as well.

\paragraph{Action of the full group average.}
For the trivial irreducible representation $\lambda=\mathrm{triv}$, we have
$\rho_{\mathrm{triv}}(g)=1$ for all $g\in G$, so
\[
\Pi_G^{(\mathrm{triv})}
=
\frac{1}{|G|}\sum_{g\in G} 1
=
1.
\]
For any nontrivial irreducible representation $\lambda\neq\mathrm{triv}$,
orthogonality of matrix coefficients implies
\[
\frac{1}{|G|}\sum_{g\in G}\rho_\lambda(g)=0.
\]
Therefore, $\Pi_G$ acts as the identity on the trivial isotypic component
$\mathcal F^G$ and annihilates all nontrivial isotypic components.

\paragraph{Consequences of exact invariance.}
Suppose now that $\Pi_S=\Pi_G$ as operators on $\mathcal F$.
Comparing the action of $\Pi_S$ and $\Pi_G$ on each irreducible block, we conclude
that for every nontrivial irreducible representation $\lambda\neq\mathrm{triv}$,
\begin{equation}
\frac{1}{|S|}\sum_{g\in S}\rho_\lambda(g)=0.
\label{eq:nontrivial-vanishing}
\end{equation}

\paragraph{Fourier-analytic interpretation.}
Define the probability measure $\mu_S$ on $G$ by
\[
\mu_S(g):=\frac{m_S(g)}{|S|},
\]
where $m_S(g)$ denotes the multiplicity of $g$ in the multiset $S$.
Equation~\eqref{eq:nontrivial-vanishing} states precisely that the Fourier transform
of $\mu_S$ vanishes on all nontrivial irreducible representations:
\[
\widehat{\mu_S}(\lambda)
:=
\sum_{g\in G}\mu_S(g)\rho_\lambda(g)
=
0
\qquad
\text{for all }\lambda\neq\mathrm{triv}.
\]

By the Fourier inversion theorem for finite groups, the only probability measure
on $G$ whose Fourier coefficients vanish on all nontrivial irreducible
representations is the uniform measure.
Hence $\mu_S(g)=1/|G|$ for all $g\in G$.

\paragraph{Conclusion.}
Therefore, the multiplicity function $m_S(g)$ is constant over $G$, meaning that
$S$ is uniform over the group.
In particular, if $S$ is a subset without repetitions, this forces $S=G$.
This completes the proof.
\end{proof}

\section{Extensions to Ordinary Least Squares (OLS) and Infinite-Dimensional Hypothesis Classes}
\label{app:ols-and-infinite-dimensional}

While in the main text, we focused on projection estimators, this choice is not
essential to data augmentation, and it only makes the statistical and
representation-theoretic effects of augmentation especially transparent. In this
section, we compare the projection estimator with ordinary least squares (OLS),
and then explain how the same finite-dimensional analysis applies to
finite-dimensional truncations of infinite-dimensional function classes.

\subsection{Projection Estimators Versus Ordinary Least Squares (OLS)}
\label{app:ols-comparison}

Let \(\mathcal F\) be an \(r\)-dimensional subspace of
\(L^2(\mathcal X,\mu)\), and let
\[
    \phi(x)
    =
    \big(\phi_1(x),\ldots,\phi_r(x)\big)^\top
\]
be an orthonormal basis of \(\mathcal F\). For
\(f_\beta(x)=\langle \beta,\phi(x)\rangle\), orthonormality gives
\[
    \|f_\beta\|_{L^2(\mathcal X)}^2
    =
    \|\beta\|_2^2,
    \qquad
    \mathbb E_{x\sim\mu}\big[\phi(x)\phi(x)^\top\big]
    =
    I_r .
\]
Given data \((x_i,y_i)_{i=1}^n\), define
\[
    \widehat b
    :=
    \frac1n\sum_{i=1}^n y_i\phi(x_i),
    \qquad
    \widehat\Sigma
    :=
    \frac1n\sum_{i=1}^n \phi(x_i)\phi(x_i)^\top .
\]
The projection estimator used throughout the paper is
\[
    \widehat\beta_{\mathrm{proj}}
    =
    \widehat b,
    \qquad
    \widehat f_{\mathrm{proj}}(x)
    =
    \langle \widehat b,\phi(x)\rangle .
\]
Thus, the projection estimator replaces the empirical covariance
\(\widehat\Sigma\) by its population value \(I_r\). This is natural in our
setting because the samples are drawn from \(\mu\), and the basis is orthonormal
with respect to \(\mu\).

By contrast, ordinary least squares solves the empirical least-squares problem
\[
    \widehat\beta_{\mathrm{ols}}
    \in
    \operatorname{arg}\min_{\beta\in\mathbb R^r}
    \frac1n\sum_{i=1}^n
    \big(y_i-\langle\beta,\phi(x_i)\rangle\big)^2 .
\]
Equivalently, it satisfies 
\[
    \widehat\Sigma\,\widehat\beta_{\mathrm{ols}}
    =
    \widehat b \implies \widehat\beta_{\mathrm{ols}}
    =
    \widehat\Sigma^\dagger \widehat b,
    \qquad
    \widehat f_{\mathrm{ols}}(x)
    =
    \langle \widehat\Sigma^\dagger \widehat b,\phi(x)\rangle,
\]
where \(\dagger\) denotes the Moore--Penrose pseudoinverse; if
\(\widehat\Sigma\) is invertible, this reduces to
\(\widehat\beta_{\mathrm{ols}}=\widehat\Sigma^{-1}\widehat b\).
Finally, let \(\Phi\in\mathbb R^{n\times r}\) be the feature matrix with rows
\(\phi(x_i)^\top\), and let \(y=(y_1,\ldots,y_n)^\top\). Then
\[
    \widehat\beta_{\mathrm{proj}}
    =
    \frac1n\Phi^\top y,
    \qquad
    \widehat\beta_{\mathrm{ols}}
    =
    \left(\frac1n\Phi^\top\Phi\right)^\dagger
    \frac1n\Phi^\top y
    =
    (\Phi^\top\Phi)^\dagger\Phi^\top y .
\]
The difference between the two estimators is therefore exactly the covariance
used to map empirical moments to coefficients: OLS uses the empirical covariance
\(\widehat\Sigma=(1/n)\Phi^\top\Phi\), whereas the projection estimator uses the
population covariance \(I_r\).

We next describe how augmentation enters these formulas. Let \(G\) act on
\(\mathcal X\), and let \(\rho:G\to \mathbb R^{r\times r}\) denote the induced
representation on \(\mathcal F\), defined by
\[
    (\rho(g)f)(x)=f(g^{-1}x).
\]
Equivalently, for \(f_\beta(x)=\langle \beta,\phi(x)\rangle\), we write
\(\rho(g)f_\beta=f_{\rho(g)\beta}\), which implies
\[
    \phi(g^{-1}x)=\rho(g)^\top \phi(x).
\]
For an augmentation set \(S\subseteq G\), the augmented empirical moment vector is
\[
    \widehat b_S
    :=
    \frac1{|S|}\sum_{s\in S} \rho(s)\widehat b .
\]
This is exactly the partial averaging operator
applied to the coefficient vector:
\[
    \widehat b_S
    =
    \Pi_S\widehat b .
\]
Thus, the augmented projection estimator is
\[
    \widehat\beta_{\mathrm{proj},S}
    =
    \Pi_S\widehat b .
\]

For ordinary least squares, augmentation also modifies the empirical covariance.
The augmented covariance is
\[
    \widehat\Sigma_S
    :=
    \frac1{|S|}\sum_{s\in S}
    \rho(s)\widehat\Sigma\rho(s)^\top .
\]
Consequently, the augmented least-squares estimator is
\[
    \widehat\beta_{\mathrm{ols},S}
    =
    \widehat\Sigma_S^\dagger \widehat b_S
    =
    \left(
        \frac1{|S|}\sum_{s\in S}
        \rho(s)\widehat\Sigma\rho(s)^\top
    \right)^\dagger
    \left(
        \frac1{|S|}\sum_{s\in S}
        \rho(s)\widehat b
    \right).
\]
This formula shows that OLS has the same first-order averaging structure as the
projection estimator, but also contains a second-order averaging operation
through the empirical covariance matrix.

This distinction explains why projection estimators provide a cleaner object
for the theoretical study of data augmentation. For projection estimators, the
effect of augmentation is exactly the action of the averaging operator
\(\Pi_S\) on \(\mathcal F\). Hence, the excess risk can be read directly from
how well \(\Pi_S\) approximates the full averaging operator \(\Pi_G\). In OLS,
the estimator depends both on the averaged moment vector \(\widehat b_S\) and
on the averaged covariance matrix \(\widehat\Sigma_S\). Thus, exact averaging on
\(\mathcal F\) immediately gives exact averaging of the first-order moment term,
whereas exact equality of the full OLS estimator with its fully augmented
counterpart also involves the covariance term. Note that this second-order term is absent
from the projection estimator because the population covariance is already
fixed to \(I_r\).

Projection estimators also avoid small-sample instability. When \(n\) is
comparable to, or smaller than, \(r\), the empirical covariance
\(\widehat\Sigma\) may be singular or poorly conditioned. OLS can therefore be
unstable unless one adds regularization or assumes enough samples to guarantee
concentration of \(\widehat\Sigma\) around \(I_r\). By contrast, the projection
estimator is well-defined for every sample size and depends linearly on the
empirical moments. This isolates the effect of augmentation from the separate
issue of empirical covariance inversion.

In short, the projection estimator may be viewed as the population-covariance
analogue of OLS:
\[
    \widehat\beta_{\mathrm{proj}}
    =
    I_r^{-1}\widehat b,
    \qquad
    \widehat\beta_{\mathrm{ols}}
    =
    \widehat\Sigma^\dagger\widehat b .
\]
Since our goal is to understand how partial augmentation suppresses
non-invariant components, projection estimators are a natural theoretical
choice: they turn augmentation into an averaging operator on \(\mathcal F\),
give clean rates, and avoid conditioning assumptions that are orthogonal to the
main phenomenon.

\subsection{Infinite-Dimensional Hypothesis Classes}
\label{app:infinite-dimensional}

The main results are stated for finite-dimensional spaces \(\mathcal F\). This
assumption separates the effect of augmentation from the separate issue of
approximation error. The same conclusions apply to infinite-dimensional classes
after choosing a finite-dimensional truncation.

Let \(\mathcal F\) be an infinite-dimensional \(G\)-invariant subspace of
\(L^2(\mathcal X,\mu)\), and let
\[
    \mathcal F_r\subset \mathcal F
\]
be an \(r\)-dimensional \(G\)-invariant subspace. For example, \(\mathcal F_r\)
may be spanned by the first \(r\) basis functions in a spectral, Fourier, or
polynomial decomposition, provided the truncation is closed under the action of
\(G\). Let
\[
    r_{\mathrm{inv}}(r)
    :=
    \dim(\mathcal F_r^G)
\]
denote the dimension of the invariant subspace inside \(\mathcal F_r\).
Applying our finite-dimensional results to \(\mathcal F_r\) gives the same
bounds with \(r\) replaced by \(\dim(\mathcal F_r)\), and with
\(r_{\mathrm{inv}}\) replaced by \(r_{\mathrm{inv}}(r)\). In particular, for a
fixed truncation \(\mathcal F_r\), the partial augmentation term scales as
\[
    \frac{r}{n|S|},
\]
while the invariant estimation term scales as
\[
    \frac{r_{\mathrm{inv}}(r)}{n}.
\]
Thus, for this truncation, the same statistical transition occurs at
\[
    |S|
    \asymp
    \frac{r}{r_{\mathrm{inv}}(r)} .
\]

For an infinite-dimensional class, one must also account for the approximation
error incurred by restricting to \(\mathcal F_r\). Let
\[
    f_r
    :=
    \Pi_{\mathcal F_r} f^\star
\]
be the projection of the target onto the truncation. Then the total error
decomposes into an estimation part and an approximation part. Schematically, for
partial augmentation on \(\mathcal F_r\), one obtains a bound of the form
\[
    \mathbb E
    \big[
        \|\widehat f_{S,r}-f^\star\|_{L^2(\mathcal X)}^2
    \big]
    \lesssim
    \underbrace{
        \frac{r_{\mathrm{inv}}(r)}{n}
        +
        \frac{r}{n|S|}
    }_{\text{estimation and augmentation}}
    +
    \underbrace{
        \|f^\star-f_r\|_{L^2(\mathcal X)}^2
    }_{\text{approximation}},
\]
up to the problem-dependent constants appearing in the finite-dimensional
results. If the target is invariant and the truncation \(\mathcal F_r\) is
\(G\)-invariant, then \(f_r\) is also invariant, so the same invariant-subspace
analysis applies inside \(\mathcal F_r\).

The remaining task is to optimize over the truncation level \(r\). This requires
balancing the statistical terms
\[
    \frac{r_{\mathrm{inv}}(r)}{n}
    +
    \frac{r}{n|S|}
\]
against the approximation error
\[
    \|f^\star-\Pi_{\mathcal F_r}f^\star\|_{L^2(\mathcal X)}^2 .
\]
The optimal truncation is problem-specific. It depends on the smoothness or
spectral decay of \(f^\star\), on how the group action decomposes across the
basis functions, and on how quickly the invariant dimension
\(r_{\mathrm{inv}}(r)\) grows with \(r\). These approximation-theoretic
questions are important, but they are separate from the main focus of this
paper. Our results characterize the effect of partial augmentation once a
finite-dimensional representation has been chosen; optimizing this
representation for a particular infinite-dimensional model is left to
problem-specific work.

\end{document}